\documentclass[twoside,11pt]{article}
\usepackage{hyperref}
\usepackage{url}

\usepackage{jmlr2e}

\usepackage{amsmath,amssymb}
\usepackage{amsbsy}
\usepackage{bm}
\usepackage[dvipdfmx]{graphicx} 

\allowdisplaybreaks

\usepackage[scriptsize]{subfigure}

\newcommand{\EE}{{\mathrm{E}}}

\newcommand{\Eqref}[1]{Eq. (\ref{#1})}

\newcommand{\dd}{\mathrm{d}}

\newcommand{\calF}{\mathcal{F}}
\newcommand{\calA}{\mathcal{A}}
\newcommand{\calU}{\mathcal{U}}
\newcommand{\boldU}{\boldsymbol{U}}

\newcommand{\sigmap}[1]{\sigma_{\mathrm{p},#1}}
\newcommand{\sigmapo}{\sigma_{\mathrm{p}}}
\newcommand{\cepsilon}{c_{\xi}}

\newcommand{\boldone}{\boldsymbol{1}}
\newcommand{\boldzero}{\boldsymbol{0}}
\newcommand{\Real}{\mathbb{R}}

\newcommand{\LPi}{L_2(P(X))}

\newcommand{\Tr}{\mathrm{Tr}}

\newcommand{\dmax}{d_{\max}}

\newtheorem{Theorem}{Theorem}
\newtheorem{Lemma}{Lemma}

\newtheorem{Assumption}{Assumption}
\newtheorem{Example}{Example}

\newcommand{\Eqrefs}[1]{Eqs. (\ref{#1})}

\title{Convergence rate of Bayesian tensor estimator: \\ Optimal rate without 
restricted strong convexity}

\author{
\name Taiji Suzuki 
\email s-taiji@is.titech.ac.jp 
 \\
\addr Department of Mathematical and Computer Sciences\\
Tokyo Institute of Technology\\
O-okayama 2-12-1, Meguro-ku, Tokyo 152-8552, JAPAN \\
}

\begin{document}


\maketitle

\begin{abstract}
In this paper, we investigate the statistical convergence rate of a Bayesian low-rank tensor
estimator.
Our problem setting is the regression problem where 
a tensor structure underlying the data is estimated.
This problem setting occurs in many practical applications, such as 
collaborative filtering, multi-task learning, and spatio-temporal data analysis.
The convergence rate is analyzed in terms of both in-sample and out-of-sample predictive accuracies.
It is shown that a near optimal rate is achieved without any strong convexity 
of the observation.
Moreover, we show that the method has adaptivity to the unknown rank of the true tensor,
that is, the near optimal rate depending on the true rank is achieved even if it is not known a priori.	
\end{abstract}


\section{Introduction}
Tensor modeling is a powerful tool for representing higher order relations between
several data sources.
The second order correlation has been a main tool in data analysis for a long time. 
However, because of an increase in the variety of data types, 
we frequently encounter a situation where higher order correlations are important for  
transferring information between more than two data sources.
In this situation, a tensor structure is required. 
For example, a recommendation system with a three-mode table, such as 
user $\times$ movie $\times$ context, is regarded as 
comprising the tensor data analysis \citep{RECSYS:Karatzoglou+etal:2010}.
The noteworthy success of tensor data analysis is based on 
the notion of the {\it low rank} property of a tensor, which is analogous to that of a matrix.
The rank of a tensor is defined by a generalized version of the singular value decomposition for matrices. 
This enables us to decompose a tensor into a few factors and 
find higher order relations between several data sources.

A naive approach to computing tensor decomposition requires 
non-convex optimization \citep{SIAMREVIEW:Kolda+Bader:2009}. 
Several authors have proposed convex relaxation methods
to overcome the computational difficulty caused by non-convexity
 \citep{ICCV:Liu+etal:2009,TechRepo:Signoretto+etal:2010,Inverse:Gandy+etal:2011,NIPS:Tomioka+etal:2011,NIPS:Tomioka+Suzuki:2013}.
The main idea of convex relaxations is to unfold a tensor into a matrix,
and apply trace norm regularization to the matrix thus obtained.
This technique connects low rank tensor estimation to the well-investigated convex low rank matrix estimation.
Thus, we can apply the techniques developed in low rank matrix estimation 
in terms of optimization and statistical theories.
To address the theoretical aspects, the authors of \cite{NIPS:Tomioka+etal:2011} gave 
the statistical convergence rate of a convex tensor estimator that 
utilizes the so-called {\it overlapped Schatten 1-norm}
defined by the sum of the trace norms of all unfolded matricizations.
In \cite{ICML:Mu+etal:2014}, it was showed that the bound given by \cite{NIPS:Tomioka+etal:2011} is tight,
but can be improved by a modified technique called {\it square deal}.
The authors of \cite{NIPS:Tomioka+Suzuki:2013} proposed 
another approach called {\it latent Schatten 1-norm} regularization that is defined by 
the infimum convolution of trace norms of all unfolded matricizations,
and analyzed its convergence rate.
These theoretical studies revealed the qualitative dependence of learning rates
on the rank and size of the underlying tensor.
However, one problem of convex methods is that 
reducing the problem to a matrix one
may lead to computational efficiency at the expense of statistical efficiency.

The main question addressed in this paper is whether we can obtain a {\it computationally tractable} method 
that possesses a better learning rate than existing methods.
To answer this question, we consider a Bayesian learning method.
Bayesian tensor learning methods have been studied extensively 
\citep{AISTATS:Chu+Zoubin:2009,TPAMI:Xu:2013,SDM:Lxiong:2010,ICML:Rai+etal:2014}.
Basically, they construct a generative model of the tensor decomposition 
and place a prior probability on the decomposed components.
As in convex methods, Bayesian methods have also been polished to 
efficiently process large datasets. 
Their statistical performances have been supported numerically,
but only a few theoretical analyses have been provided. 

In this paper, we present the learning rate of a Bayes estimator for low-rank tensor regression problems.
The problem occurs in many applications, such as collaborative filtering, multi-task learning,
spatio-temporal data analysis.
The prior probability we consider here is the most basic one, 
which places Gaussian priors on decomposed components and 
an exponentially decaying prior on the rank (see \cite{SDM:Lxiong:2010,TPAMI:Xu:2013,ICML:Rai+etal:2014}).
Roughly speaking, we obtain the (near) optimal convergence rate,
$$
\|\hat{A}-A^*\|_n^2 = O_p\left(\frac{d^*(\sum_{k=1}^K M_k)\log(K\sqrt{n(\sum_k M_k)^K}}{n}\right),
$$
where $n$ is the sample size, $\hat{A}$ is the Bayes estimator, $A^*$ is the true tensor, 
$d^*$ is the {\it CP-rank} of the true tensor (its definition will be given in Section \ref{sec:ProbSetting}), and 
$(M_1,\dots,M_K)$ is the size.
Moreover, our analysis has the following favorable properties.
\begin{itemize}
\item The near optimal rate is shown {\it without} assuming any strong convexity on the empirical $L_2$ norm. 
\item Rank adaptivity is shown, that is, the convergence rate 
is automatically  adjusted to the rank of the true tensor as if we knew it a priori. 
\end{itemize}
In particular, the first property significantly differentiates our approach from existing approaches.
A variant of strong convexity, such as restricted strong convexity \citep{AS:Bickel+etal:2009,SS:Negahban+etal:2012}, is usually assumed 
in convex sparse learning in order to derive a fast rate.
However, for the analysis of the predictive accuracy of a Bayes estimator,
a near optimal rate can be shown without such conditions.

To the best of our knowledge, this is the first result that shows  
 (near) optimality of a computationally tractable low-rank tensor estimator.

\section{Problem Settings}
\label{sec:ProbSetting}
In this section, the problem setting of this paper is shown.
Suppose that there exists the true tensor $A^*\in \Real^{M_1\times \dots \times M_K}$
of order $K$,
and we observe $n$ samples $D_n=\{(Y_i,X_i)\}_{i=1}^n$ from the following linear model:
$$
Y_i = \langle A^*, X_i \rangle + \epsilon_i.
$$
Here, $X_i$ is a tensor in $\Real^{M_1\times \dots \times M_K}$ and 
the inner product $\langle \cdot,\cdot \rangle$ is defined by
$\langle A, X \rangle = \sum_{j_1,\dots,j_K=1}^{M_1,\dots,M_K} A_{j_1,\dots,j_K} X_{j_1,\dots,j_K}$
between two tensors $A,X \in \Real^{M_1\times \dots \times M_K}$.
$\epsilon_i$ is i.i.d. noise from a normal distribution $N(0,\sigma^2)$ with mean $0$ and variance $\sigma^2$.

Now, we assume the true tensor $A^*$ is ``low-rank.''
The notion of rank considered in this paper is {\it CP-rank} (Canonical Polyadic rank) \citep{JMP:Hitchcock:1927,JMP:Hitchcock:1927:2}.
We say a tensor $A\in \Real^{M_1\times \dots \times M_K}$ has CP-rank $d'$
if there exist matrices $U^{(k)} \in \Real^{d'\times M_{k}}~(k=1,\dots,K)$
such that  
$
A_{j_1,\dots,j_K} =
\sum_{r=1}^{d'} U^{(1)}_{r,j_1}U^{(2)}_{r,j_2}\dots U^{(K)}_{r,j_K},
$
and $d'$ is the minimum number to yield this decomposition.
This is called {\it CP-decomposition}. 
When $A$ satisfies this relation for ${\boldsymbol{U}}=(U^{(1)},U^{(2)},\dots,U^{(K)})$, we write 
\begin{align}
A = A_{\boldU} =   [[U^{(1)},U^{(2)},\dots,U^{(K)} ]]
=: \left(\sum\nolimits_{r=1}^{d'} U^{(1)}_{r,j_1}U^{(2)}_{r,j_2}\dots U^{(K)}_{r,j_K}\right)_{j_1,\dots,j_K}.
\label{eq:CPdecomp}
\end{align}
We denote by $d^*$ the CP-rank of the true tensor $A^*$.

In this paper, we investigate the predictive accuracy of the linear model with the assumption that 
$A^*$ has  low CP-rank.
Because of the low CP-rank assumption, the learning problem becomes more structured than
an ordinary linear regression problem on a vector.
There are two types of predictive accuracy: {\it in-sample} one and {\it out-of-sample}.
The in-sample predictive accuracy of an estimator $\hat{A}$ is defined by 
\begin{equation}
\|\hat{A}-A^*\|_n^2 := \frac{1}{n} \sum_{i=1}^n \langle X_i,\hat{A} - A^*\rangle^2,
\label{eq:InSamplePA}
\end{equation}
where $\{X_i\}_{i=1}^n$ is the observed input samples. 
The out-of-sample one is defined by 
\begin{equation}
\|\hat{A}-A^*\|_{\LPi}^2 := \EE_{X\sim P(X)}[ \langle X, \hat{A}- A^*\rangle^2],
\label{eq:OutOfSamplePA}
\end{equation}
where $P(X)$ is the distribution of $X$ that generates the observed samples $\{X_i\}_{i=1}^n$
and the expectation is taken over independent realization $X$ from the observed ones. 

\begin{Example}
{\bf Tensor completion under random sampling.}
\label{ex:tensorcompletion}
Suppose that we have partial observations of a tensor.
A tensor completion problem consists of denoising the observational noise and completing the unobserved elements.
In this problem, 
$X_i$ is independently identically distributed from a set $\{e_{j_1,\dots,j_K}  \mid 1\leq j_k \leq M_k~(k=1,\dots,K)\},$
where $e_{j_1,\dots,j_K}$ is an indicator vector that has 1 at its $(j_1,\dots,j_K)$-element and 0 elsewhere,
and thus, $Y_i$ is an observation of one element of $A^*$ contaminated with noise $\epsilon_i$.

The out-of-sample accuracy measures how accurately we can recover the underlying tensor $A^*$ from the 
partial observation. 
If $X_i$ is uniformly distributed,
$\|\hat{A}-A^*\|_{\LPi}^2 = \frac{1}{M_1 \dots M_K} 
\|\hat{A} -A^* \|_2^2$, where $\|\cdot\|_2$ is the $\ell_2$-norm obtained by summing 
the squares of all the elements.
If $K=2$, this problem is reduced to the standard matrix completion problem.
In that sense, our problem setting is a wide generalization of the low rank matrix completion problem.
\end{Example}
\begin{Example}
{\bf Multi-task learning.}
Suppose that several tasks are aligned across a 2-dimensional space.
For each task $(s,t)  \in \{1,\dots,M_1\}\times \{1,\dots,M_2\}$
(indexed by two numbers), 
there is a true weight vector $a^*_{(s,t)} \in \Real^{M_3}$.
The tensor $A^*$ is an array of the weight vectors $a^*_{(s,t)}$, that is,
$A^*_{s,t,j} = a^*_{(s,t),j}$.

The input vector $X_i$ is a vector of predictor variables for one specific task, say $(s,t)$, and 
takes a form such that 
$
X_{i,(s',t',:)} = \begin{cases} x_{i}^{(s,t)} \in \Real^{M_3}, &((s',t') = (s,t)),\\ 
\boldzero,&(\text{otherwise}).\end{cases}
$

By assuming $A^*$ is low-rank in the sense of CP-rank, 
the problem becomes a multi-task feature learning with a two dimensional structure in the task space
 \citep{ICML:Bernardino+etal:2013}.

\end{Example}

As shown in the examples, the estimation problem of low-rank tensor $A^*$ 
is a natural extension of low-rank matrix estimation.
However, it has a much richer structure than matrix estimation.
Thus far, some convex regularized learning problems have been proposed 
analogously to spectrum regularization on a matrix, 
and their theoretical analysis has also been provided.
However, no method have been proved to be statistically optimal.
There is a huge gap between a matrix and higher order array.
One reason for this gap is the computational complexity of the convex envelope of CP-rank.
It is well known that the trace norm of a matrix is a convex envelope of the matrix rank on a set of 
matrices with a restricted operator norm~\citep{NIPS:Srebro+Rennie+Jaakkola:2005}.
However, as for CP-rank, it becomes NP-hard to compute its convex envelope if $K \geq 3$.

In this paper, we present a Bayes estimator instead of the convex regularized one.
It will be shown that our Bayes estimator shows a near optimal convergence rate 
with a much weaker assumption, while the learning procedure is computationally tractable.
The rate is much improved as compared to that of the existing estimators.

\section{Bayesian tensor estimator}
We now provide the prior distribution of the Bayes estimator that is investigated in this paper.
On a decomposition of a rank $d'$ tensor $A=[[U^{(1)},U^{(2)},\dots,U^{(K)} ]]~(U^{(k)}\in \Real^{d'\times M_k})$, 
we place a Gaussian prior:
\begin{align}
\pi(U^{(1)},\dots,U^{(K)}|d') 
&
\propto \exp\bigg\{ -\frac{d'}{2 \sigma_{\mathrm{p}}^2}\sum_{k=1}^K \Tr[{U^{(k)}}^\top U^{(k)}] \bigg\}, 
\end{align}
where $\sigmapo > 0$.
Moreover, we placed a prior distribution on the rank $1 \leq d' \leq \dmax$ as
$$
\pi(d') = 
  {\textstyle \frac{1}{N_\xi}} \xi^{d'(M_1+\dots+M_K)},
$$
where $0 < \xi < 1$ is some positive real number, 
$\dmax$ is a sufficiently large number that is supposed to be larger than $d^*$,
and $N_{\xi}$ is the normalizing constant, $N_\xi = \frac{1-\xi^{M_1+\dots+M_K}}{\xi-\xi^{\dmax(M_1+\dots+M_K)}}$.

Now, since the noise is Gaussian, the likelihood of a tensor $A$ is given by 
$$p(D_n|A) =: p_{n,A} \propto \exp\left\{- \frac{1}{2\sigma^2} \sum_{i=1}^n (Y_i - \langle A, X_i \rangle)^2\right\}.$$

The posterior distribution is given by 
$$
\Pi(A \in \mathcal{C} | D_n)
= \frac{\sum_{d=1}^{\dmax}
\int_{A_{\boldU} \in \mathcal{C}} p_{n,A_{\boldU}} \pi(U^{(1)},\dots,U^{(K)}|d) \pi(d)  
\dd U^{(1)}\dots \dd U^{(K)}
}{
\sum_{d=1}^{\dmax}
\int  p_{n,A_{\boldU}} \pi(U^{(1)},\dots,U^{(K)}|d) \pi(d)  \dd U^{(1)}\dots \dd U^{(K)}
},
$$
where $\mathcal{C} \subseteq \Real^{M_1 \times \dots M_K}$.
It is noteworthy that
the posterior distribution of $U^{(k)}$ conditioned by $d$ and $U^{(k')}~(k' \neq k)$
is a Gaussian distribution, because the prior is conjugate to Gaussian distributions. 
Therefore, the posterior mean $\int_A f(A) \Pi(\dd A | D_n)$ of a function $f:\Real^{M_1 \times \dots M_K}$ 
can be computed by an MCMC method, such as Gibbs sampling \citep{SDM:Lxiong:2010,TPAMI:Xu:2013,ICML:Rai+etal:2014}.
In this paper, we consider the posterior mean estimator $\hat{A} = \int A \Pi(\dd A | D_n)$ which 
is the Bayes estimator corresponding to the square loss.

\section{Convergence rate analysis}

In this section, we present the statistical convergence rate of the Bayes estimator.
Before we give the convergence rate, we define some quantities and provide the assumptions. 
We define the {\it max-norm} of $A^*$ as
$$
\|A^*\|_{\max,2} := \min_{\{U^{(k)}\}}\{ \max_{i,k} \|U^{(k)}_{:,i}\|_2 \mid A^* = [[U^{(1)},\dots,U^{(K)}]],~U^{(k)} \in \Real^{d^* \times M_{k}} \}.
$$
The $\ell_p$-norm of a tensor $A$ is given by $\|A\|_p := (\sum_{j_1,\dots,j_K}
|A_{j_1,\dots,j_K}|^{p})^{\frac{1}{p}}$.
The prior mass around the true tensor $A^*$ is a key quantity for characterizing the convergence rate,
which is denoted by $\Xi$:
$$
\Xi(\delta) := -\log(\Pi(A:\|A-A^*\|_n < \delta)),
$$
where $\delta > 0$.
Small $\Xi(\delta)$ means that the prior is well concentrated around the truth.
Thus, it is natural to consider that, if $\Xi$ is large, the Bayes estimator could be close to the truth.
However, clearly, we do not know beforehand the location of the truth. 
Thus, it is not beneficial to place too much prior mass around one specific point. 
Instead, the prior mass should cover a wide range of possibilities of $A^*$.
The balance between concentration and dispersion has a similar meaning to that of the bias-variance trade-off.

To normalize the scale, we assume that the $\ell_1$-norm of $X_i$ is bounded.
\begin{Assumption}
\label{ass:Xl1bound}
We assume that the $\ell_1$-norm of $X_i$ is bounded by 1:
$
\|X_i\|_1 \leq 1,~~\text{a.s.}.
$\footnote{$\ell_1$-norm could be replaced by another norm such as $\ell_2$.
This difference affects the analysis of out-of-sample accuracies, but rejecting 
samples with $\max_X |\langle X, A\rangle|\leq R$ gives an analogous result for other norms.}
\end{Assumption}
For theoretical simplicity, we utilize the unnormalized prior on the rank, that is $\pi(d) = \xi^{d(\sum_{k}M_k)}$
\footnote{This is not essential, but for just making the expression of $\Xi$ simple.}.
It should be noted that removing the normalization constant does not affect the posterior.
Under these assumptions, $\Xi(\delta)$ can be bounded as follows.
\begin{Lemma}
Under Assumption \ref{ass:Xl1bound},
the prior mass $\Xi$ has the bound
\begin{align*}
\Xi({\textstyle \frac{r}{\sqrt{n}}}) \leq &
d^* \left(\sum_{k=1}^K M_k \right)  \log\left[ \frac{6}{\xi}
\bigg( {\textstyle \frac{\sqrt{n}\sigmapo^K K\left(\frac{\|A^*\|_{\max,2}}{\sigmapo}+1\right)^{K-1}}{r}} \vee 1\bigg) \right] 
+ \frac{d^* \sum_{k=1}^K  \|U^{*(k)}\|_{F}^2 }{2\sigmapo^2}
\end{align*}
for all $r >0$, where $\{U^{*(k)}\}_k$ are any tensors satisfying $A^* = [[U^{*(1)},\dots,U^{*(K)}]]$.
\end{Lemma}

It should be noted that $\Xi(r/\sqrt{n}) \leq \Xi(1/\sqrt{n})$ for all $r>1$.
Finally, we define the technical quantities
$C_{n,K} := 3K\sqrt{n}
\left(\frac{4\sigmapo^2 \Xi(\frac{1}{\sqrt{n}})}{d^*}\right)^{\frac{K}{2}}$
and $\cepsilon := \min\{ |\log(\xi)|/\log(C_{n,K}),1 \}/4$.

\subsection{In-sample predictive accuracy}
We now give the convergence rate of the in-sample predictive accuracy.
We suppose the inputs $\{X_i\}_{i=1}^n$ are fixed (not random).
The in-sample predictive accuracy conditioned by $\{X_i\}_{i=1}^n$ is given as follows.
\begin{Theorem}
\label{th:insamplerate}
Under Assumption \ref{ass:Xl1bound},
there exists a universal constant  $C$ such that 
the posterior mean of the in-sample accuracy is upper bounded by
\begin{align}
\EE\left[\int\|A-A^*\|_n^2 \dd \Pi(A|Y_{1:n})\right] 
\leq 
\frac{C}{n}
\Bigg[
&
d^*\left(\sum\nolimits_k M_k + \frac{1}{|\log(\xi)|}\right)\log(C_{n,K}) 
+
\frac{\Xi({\small \sqrt{\frac{\cepsilon}{n}}})}{\cepsilon} 
\notag \\
&
+ \log(\dmax) + K
+
8^K (K+1)!
\Bigg].
\label{eq:AAstarIntBound_compact}
\end{align}
\end{Theorem}
The proof is given in Appendix \ref{sec:ProofTh1}.
This theorem provides the speed at which the posterior mass concentrates around the true $A^*$.
It should be noted that the integral in the LHS of \Eqref{eq:AAstarIntBound_compact}
is taken {\it outside} $\|A-A^*\|_n^2$.
This gives not only information about the posterior concentration but also
the convergence rate of the posterior mean estimator. 
This can be shown as follows. By Jensen's inequality, we have 
$\EE\left[\|\int A \dd \Pi(A|Y_{1:n})-A^*\|_n^2 \right]  \leq  \EE\left[\int\|A-A^*\|_n^2 \dd \Pi(A|Y_{1:n})\right] $.
Therefore, Theorem \ref{th:insamplerate} gives a much stronger claim on the posterior than 
just stating the convergence rate of the posterior mean estimator.

Since the rate \eqref{eq:AAstarIntBound_compact} is rather complicated, we give a simplified bound.
By assuming $\log(\dmax)$ and $K!$ are smaller than $d^*(\sum_{k} M_k)$, we rearrange it as 
$$
\EE\left[\int\|A-A^*\|_n^2 \dd \Pi(A|Y_{1:n})\right] 
= O\left(\frac{ d^*(M_1 + \dots  + M_K)}{n}
\log\bigg(K\sqrt{n ({\textstyle \sum_{k=1}^K M_k})^K} \frac{\sigmapo^{K}}{\xi}\bigg) \right).
$$
Inside the $O(\cdot)$ symbol, a constant factor depending on $K,\|A^*\|_{\max,2},\sigmapo,\xi$ 
is hidden.
This bound means that the convergence rate is characterized by the 
actual degree of freedom up to a log term.
That is, since the true tensor has rank $d^*$ and thus has a decomposition \eqref{eq:CPdecomp}, 
the number of unknown parameters is bounded by $d^*(M_1 + \dots  + M_K)$.
Thus, the rate is basically $O(\frac{\text{degree of freedom}}{n})$ (up to $\log$ order), which is optimal.
Here, we would like to emphasize that the true rank $d^*$ is unknown, but by placing a prior distribution 
on a rank the Bayes estimator can appropriately estimate the rank and gives an almost optimal rate 
depending on the true rank. 
In this sense, the Bayes estimator has adaptivity to the true rank.

More importantly, we do not assume {\it any strong convexity} on the design.
Usually, to derive a fast convergence rate of sparse estimators, such as Lasso and the trace norm regularization estimator,
we assume a variant of strong convexity, such as a restricted eigenvalue condition \citep{AS:Bickel+etal:2009}
and restricted strong convexity \citep{SS:Negahban+etal:2012}.
However, our convergence rate {\it does not} require such conditions.
This is a quite strong point in our analysis.
One reason why this is possible is that we are interested in the predictive accuracies rather than
the actual distance between the tensors $\|A-A^*\|_2^2$.
It is difficult to check the strong convexity condition in practice.
However, we do not need to consider this when we are interested only in the predictive accuracy.

\subsection{Out-of-sample predictive accuracy}
Next, we turn to the convergence rate of the out-of-sample predictive accuracy.
In this setting, the input sequence $\{X_i\}_{i=1}^n$ is not fixed,
but an i.i.d. random variable generated by a distribution $P(X)$.

To obtain fast convergence of the out-of-sample accuracy,
we need to bound the 
difference between the empirical and population $L_2$-errors: 
$\|A-A^*\|_n^2 - \|A-A^*\|_{\LPi}^2$.
To ensure that this quantity is small using Bernstein's inequality, 
$\max_X |\langle X,A \rangle|$ 
should be bounded. 
However, the infinity norm of the posterior mean could be large in tensor estimation.
This difficulty can be avoided by rejecting posterior sample $A$ with a large infinity norm.

\subsubsection{Rejection sampling with infinity norm thresholding}
Now, define $\|A\|_{\infty} = \max_{j_1,\dots,j_K}|A_{j_1,\dots,j_K}|$. 
Then, under Assumption \ref{ass:Xl1bound}, we have $\langle X ,A \rangle \leq \|A\|_{\infty}$.
Here, we assume that the infinity norm $\|A^*\|_{\infty}$ of the true tensor 
is approximately known, that is, we know $R > 0$ such that $2 \|A^*\|_{\infty} < R$.
This is usually true. For example, we know the upper bound in tensor completion for a recommendation system.
Otherwise, we may apply cross validation.
Our strategy is to reject the posterior samples with an infinity norm larger than $R$ 
during the sampling scheme.

The resultant distribution of rejection sampling is the conditional posterior distribution
$\Pi( \cdot | \|A\|_{\infty} \leq  R, D_n)$.
Accordingly, we investigate the out-of-sample accuracy of the conditional posterior:
\begin{align*}
&\EE_{D_n}\left[\int \|A-A^*\|_{\LPi}^2 \dd\Pi(A|\|A\|_{\infty} \leq R, D_n) \right].
\end{align*}
It should be noted that the conditional posterior is interpreted as the proper posterior distribution 
corresponding to the ``truncated'' prior the support of which is restricted to $\|A\|_{\infty} \leq R$.
\begin{Theorem}
Suppose Assumption \ref{ass:Xl1bound} and  $\|A^*\|_{\infty} < \frac{1}{2}R$ are satisfied,
then the out-of-sample accuracy is bounded as 
\begin{align*}
&
\EE_{D_n}\left[\int  \|A - A^* \|_{\LPi}^2 \dd\Pi(A|\|A\|_{\infty} \leq  R, D_n) \right] 
\notag \\
\leq
&
\frac{C(R^2\vee 1)}{n}
\Bigg[
d^*\left(\sum\nolimits_k M_k + \frac{3}{|\log(\xi)|}\right)\log(C_{n,K}) 
+ 
\frac{\Xi(\sqrt{\frac{\cepsilon}{n}})}{\cepsilon} +
\log(\dmax) + 
8^K (K+1)!
\Bigg],
\end{align*}
where $C$ is a universal constant.
\end{Theorem}
The proof is given in Appendix \ref{sec:ProofTh2}.
The only difference between this and Theorem \ref{th:insamplerate} is that 
$R^2$ appears in front of the bound.
The source of this factor is the gap between the empirical and population $L_2$-norms.
Here again, the convergence rate can be simplified as 
\begin{align}
&\EE_{D_n}\left[\int  \|A - A^* \|_{\LPi}^2 \dd\Pi(A|\|A\|_{\infty} \leq  R, D_n) \right] \notag \\
\leq
& O\left(\frac{ d^*(M_1 + \dots  + M_K)}{n}(R^2 \vee 1)
\log\bigg(K\sqrt{n ({\textstyle \sum_{k=1}^K M_k})^K} \frac{\sigmapo^{K}}{\xi}\bigg) \right),
\end{align}
if $K!$ and $|\log(\xi)|$ are sufficiently smaller than $d^*(\sum_k M_k)$.
Here, we observe that the convergence rate achieved is optimal up to the $\log$-term.
We would like to emphasize again that 
the optimal rate is achieved, although 
we do not assume any strong convexity on the 
distribution $L_2(\Pi)$.
This can be so because we are not analyzing the actual $L_2$-norm $\|A-A^*\|_2$.
If we do not assume strong convexity like $\|A-A^*\|_2 \leq C \|A-A^*\|_{\LPi}$,
it is impossible to derive fast convergence of $\|A-A^*\|_2$.
The trick is that we focus on the ``weighted'' $L_2$-norm $\|A-A^*\|_{\LPi}$
instead of  $\|A-A^*\|_2$.

Finally, it is remarked that, if $X_i$ is the uniform at random observation in the tensor completion problem,
then $\|A-A^*\|_{\LPi}^2 = \frac{1}{\prod_{k}M_k}\|A-A^*\|_2^2$
(note that in this setting $\|X_i\|_1 = 1$).
Thus, our analysis yields fast convergence of the tensor recovery.
If $K=2$, the analysis recovers the well known rate of matrix completion problems
up to a $\log(nM_1 M_2)$ term
\citep{JMLR:Negahban+Wainwright:2012,SS:Negahban+etal:2012,AS:Rohde+Tsybakov:2011}:
$$
\frac{1}{M_1 M_2}\|\hat{A}-A^*\|_2^2  = O_p\left(\frac{d^*(M_1 + M_2)}{n}\log(nM_1 M_2)\right).
$$

\subsubsection{Rejection sampling with max-norm thresholding}
Finally, we briefly describe the convergence rate of the Bayes estimator based on 
the rejection sampling with respect to restricted {\it max-norm}.
We reject the posterior sample with a max-norm larger than $R$;
that is, we accept only a sample
$A_{\boldU}$ that satisfies 
$\boldU \in \{(U^{(1)},\dots,U^{(K)})
\mid \|U^{(k)}_{:,j}\| \leq R~(1\leq k \leq K, 1\leq j \leq M_k)\} =: \calU_R$.
Then, we have the following bound.
\begin{Theorem}
Under  Assumption \ref{ass:Xl1bound} and  $\|A^*\|_{\max,2}+\sigmapo < R$,
we have 
\begin{align*}
& 
\EE_{D_n}\left[\int \! \|A_{\boldU} - A^*\|_{\LPi}^2 \dd \Pi(A_{\boldU} 
| \boldU \in \calU_R, D_n)\right]  \\
\leq 
&C
\Bigg(
\frac{d^*(\sum_{k=1}^K M_k)}{n}
(1\vee R^{2K})
\log\left(K\sqrt{n}R^{\frac{K}{2}}\frac{\sigmapo^{K}}{\xi}\right)
\Bigg),
\end{align*}
where $C$ is a constant depending on $K,\log(\dmax),\sigmapo$.
\end{Theorem}
The proof is given in Appendix \ref{sec:ProofTh3}.
In a setting where $M_k$ is much larger than $R$,
this bound gives a much better rate than the previous ones, 
because the term inside $\log$ is improved from $\prod M_k^{\frac{1}{2}}$ to $R^{\frac{K}{2}}$.
On the other hand, the rejection rate during the sampling would be increased.

\section{Related works}
In this section, we describe the existing works and clarify their relation to our work. 
Recently, theoretical analyses of convex regularized low-rank tensor estimators
have been developed.  
The pioneering work \citep{NIPS:Tomioka+etal:2011} analyzed
a method that utilizes unfolded matricization of a tensor.
Let $M = \prod_{k=1}^K M_k$ (the number of whole elements).
The authors used so-called {\it overlapped Schatten 1-norm} regularization  
$\sum_{k=1}^K \|A^{(k)}\|_{\Tr}$ where 
$\|\cdot\|_{\Tr}$ is the trace norm and $A^{(k)} \in \Real^{M_k \times M/M_k}$
 is the {\it mode-$k$ unfolding} of a tensor $A$ that is  a matrix obtained 
by unfolding the tensor with the $k$-th index fixed.
Their analysis assumes the true tensor $A^*$ has a low {\it Tucker-rank} \citep{Tucker:Psycometrika:1966}.
Tucker-rank is a general notion of CP-rank.
In this sense, their analysis is more general than ours.
However, strong convexity of the empirical and population $L_2$-norm is assumed.
Under this setting and $n=M$, the following bound is obtained:
\begin{align}
\label{eq:OverlapConvRate}
\frac{1}{M}\|\hat{A} - A^*\|_2^2 \leq C \frac{d^*}{n} 
{\textstyle \left(\frac{1}{K}\sum_{k=1}^K \sqrt{\frac{M}{M_k}} \right)^2}.
\end{align}
It can be seen that our bound $\frac{d^*(\sum_{k=1}^K M_k)}{n}\log(nM)$ is smaller than this bound;
in particular, if $M_k$ is large, the difference is significantly large.
In \cite{ICML:Mu+etal:2014}, it was shown that the bound 
\eqref{eq:OverlapConvRate} is tight and cannot be improved 
if the overlapped Schatten 1-norm is used.
In \cite{ICML:Mu+etal:2014}, a novel method called {\it square deal} was proposed and 
was shown to achieve the following rate. For $n=M$,
$$
\frac{1}{M}\|\hat{A} - A^*\|_2^2 \leq C \frac{d^*}{n} 
{\textstyle \left( \prod_{k \in I_1} M_k +  \prod_{k \in I_2} M_k \right)},
$$
where $I_1$ and $I_2$ are any disjoint decomposition of index set $\{1,\dots,K\}$.
This improves the rate \eqref{eq:OverlapConvRate}, but is still larger than the rate of the Bayes estimator,
because the product of $M_k$ appears instead of the sum.

Another study on the regularization approach was presented in \cite{NIPS:Tomioka+Suzuki:2013}.
The authors proposed using the {\it latent Schatten 1-norm}: 
$\inf_{\{A_k\}:A=\sum_{k=1}^K A_k}
\sum_{k=1}^K \|A_k^{(k)}\|_{\Tr}$, where $A_k^{(k)} \in \Real^{M_k \times M/M_k}$
is the mode-$k$ unfolding of the tensor $A_k$.
A nice point of this method is that it automatically finds the minimum rank direction,
that is, the mode-$k$ unfolding $A^{*(k)}$ with the minimum rank.
It was shown that the rate is 
$$
\frac{1}{M}\|\hat{A} - A^*\|_2^2 \leq C \frac{d^* \max_{k=1}^K\{M_k +  M/M_k\}}{n}.
$$
This rate is also larger than that of the Bayes estimator.
We would like to remark that this rate is obtained for the low ``Tucker-rank'' situation, which is more general than our
low CP-rank setting, and thus, it is not best suited to our situation. 
However, it is not apparent that the latent Schatten 1-norm achieves the same rate as that of 
the Bayes estimator in low CP-rank settings.

As for Bayesian counter part, 
a Bayesian low rank matrix estimator is analyzed in \cite{ALT:Alquier:2005}.
The prior in this study is similar to ours with $K=2$,
but, instead of placing prior on the rank, the authors placed a Gamma prior on the variance of the Gaussian prior.
They utilized the novel PAC-Bayes technique \citep{COLT:McAllester:1998,Book:Catoni:2004} to show
$$
\|\hat{A} - A^*\|_n^2 \leq C \frac{d^* (M_1 + M_2) \log(nM_1 M_2)}{n}.
$$
This work has a similar flavor to ours
in the sense that no strong convexity is required to obtain the convergence rate.
However, there are several points in which it differs from ours:
our analysis deals with general tensor estimation $(K\geq 3)$ and
gives a bound also on the out-of-sample predictive accuracy by utilizing the rejection sampling scheme,
and the posterior concentration is also given ($\int \|A-A^*\|_n^2\dd \Pi(A|D_n)$ instead of 
$\|\hat{A} - A^*\|_n^2$ where $\hat{A}$ is the posterior mean).
By virtue of the analysis of the out-of-sample predictive accuracy, our analysis is applicable to tensor completion
(and, of course, to matrix completion).
As for the Bayesian tensor estimator, in \cite{arXiv:Zhou+etal:BayesTensor:2013}  
a Bayes estimator of probabilistic tensors was investigated.
The model applies fully observed multinomial random variables 
and the rank is determined beforehand. 
Therefore, the setting is quite different from ours.

\vspace{-0.3cm}
\section{Numerical experiments}
\vspace{-0.3cm}

We now present numerical experiments to justify our theoretical results.
The problem is the tensor completion problem where 
each observation is a random selection of one element of $A^*$ 
with observational noise $N(0,1)$ (see Example \ref{ex:tensorcompletion}).
The true tensor $A^*$ 
was randomly generated such that
each element of $U^{(k)}~(k=1,\dots,K)$ was uniformly distributed on $[-1,1]$.
$\sigmapo$ was set at 5, and the true tensor was estimated by the posterior mean
obtained by the rejection sampling scheme with $R = 10$. 
The experiments were executed in five different settings, 
called settings 1 to 5: 
$\{(M_1,\dots,M_K),d^*\} = \{(10,10,10),4\}, 
\{(10,10,40),5\}, \{(20,20,30),8\}, 
\{(20,30,40),5\},
\{(30,30,40),6\}$.
For each setting, we repeated the experiments five times 
and computed the average of the in-sample predictive accuracy and
out-of-sample accuracy over all five repetitions.
The number of samples was chosen as $n = n_{\mathrm{s}} \prod_k M_k$,
where  $n_{\mathrm{s}}$ varied from 0.3 to 0.9.

In addition to the actual accuracy,
we considered the ``scaled'' accuracy, 
which is defined by 
$
\|\hat{A} - A^*\|_n^2 \times \left(\frac{\prod_k M_k}{d^* (\sum_k M_k)}\right);
$
the scaled out-of-sample accuracy is also defined in the same manner.
Figure \ref{fig:insample}
shows 
the in-sample accuracies and the scaled in-sample accuracies 
against the sample ratio $n_{\mathrm{s}}$.
The same plot for the out-of-sample accuracy is shown in Figure \ref{fig:outsample}.
It can be seen that the curves of the scaled accuracies in all settings are satisfactorily overlapped.
This means that our bound accurately describes the sample complexity of 
the Bayesian tensor estimator,
because according to our bounds the scaled accuracies should 
behave as $1/n_s$ up to a constant factor (and a $\log$ term).
The figures show that the scaling factor given by our theories 
is well matched to the actual predictive accuracy.

\begin{figure}[th]
\begin{minipage}{0.5\hsize}
\begin{center}
\includegraphics[width=7cm]{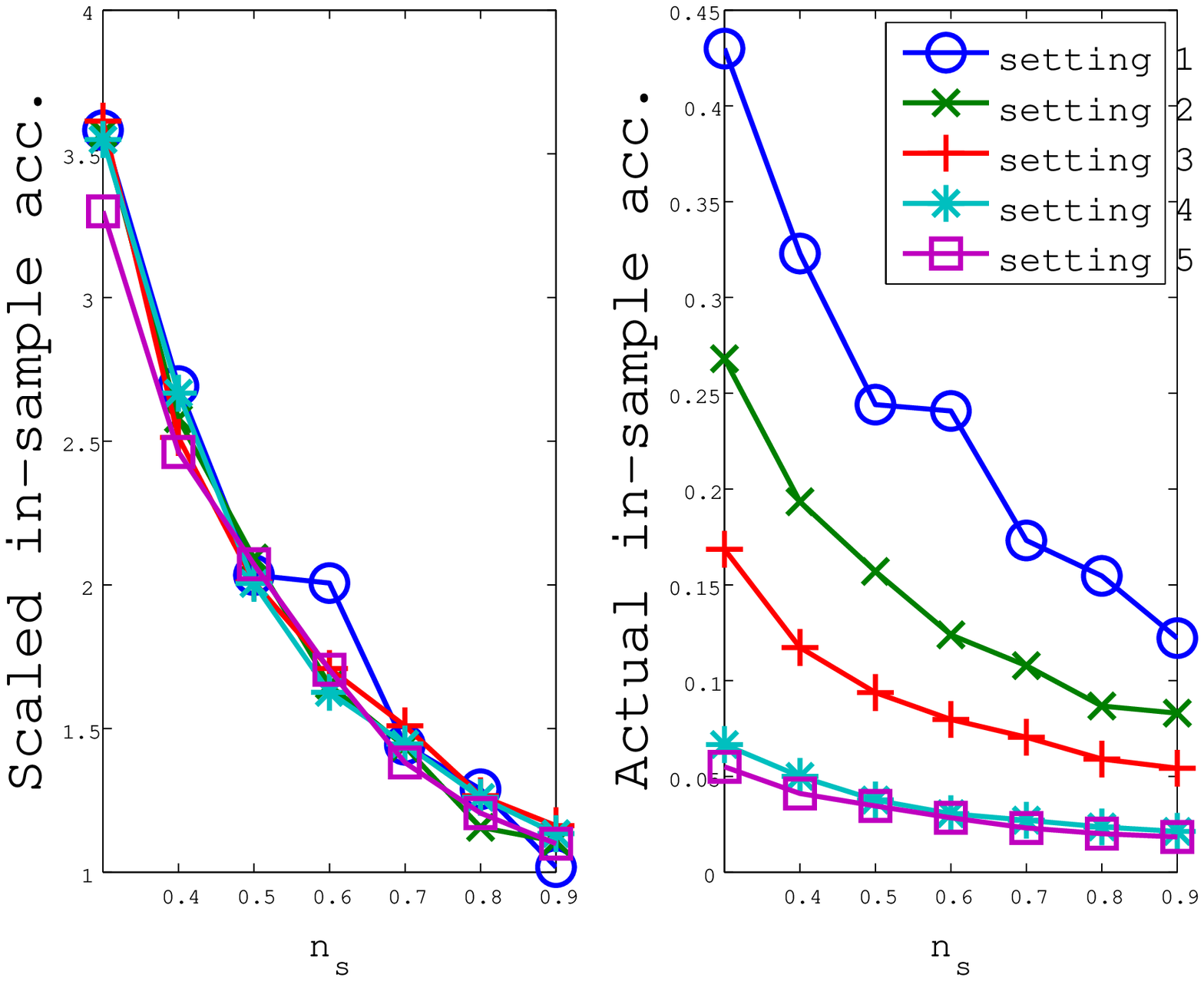}
\caption{Scaled in-sample accuracy (left) and actual in-sample accuracy (right)
versus $n_s$, averaged over five repetitions.}
\label{fig:insample}
\end{center}
\end{minipage}
~~~
\begin{minipage}{0.5\hsize}
\begin{center}
\includegraphics[width=7cm]{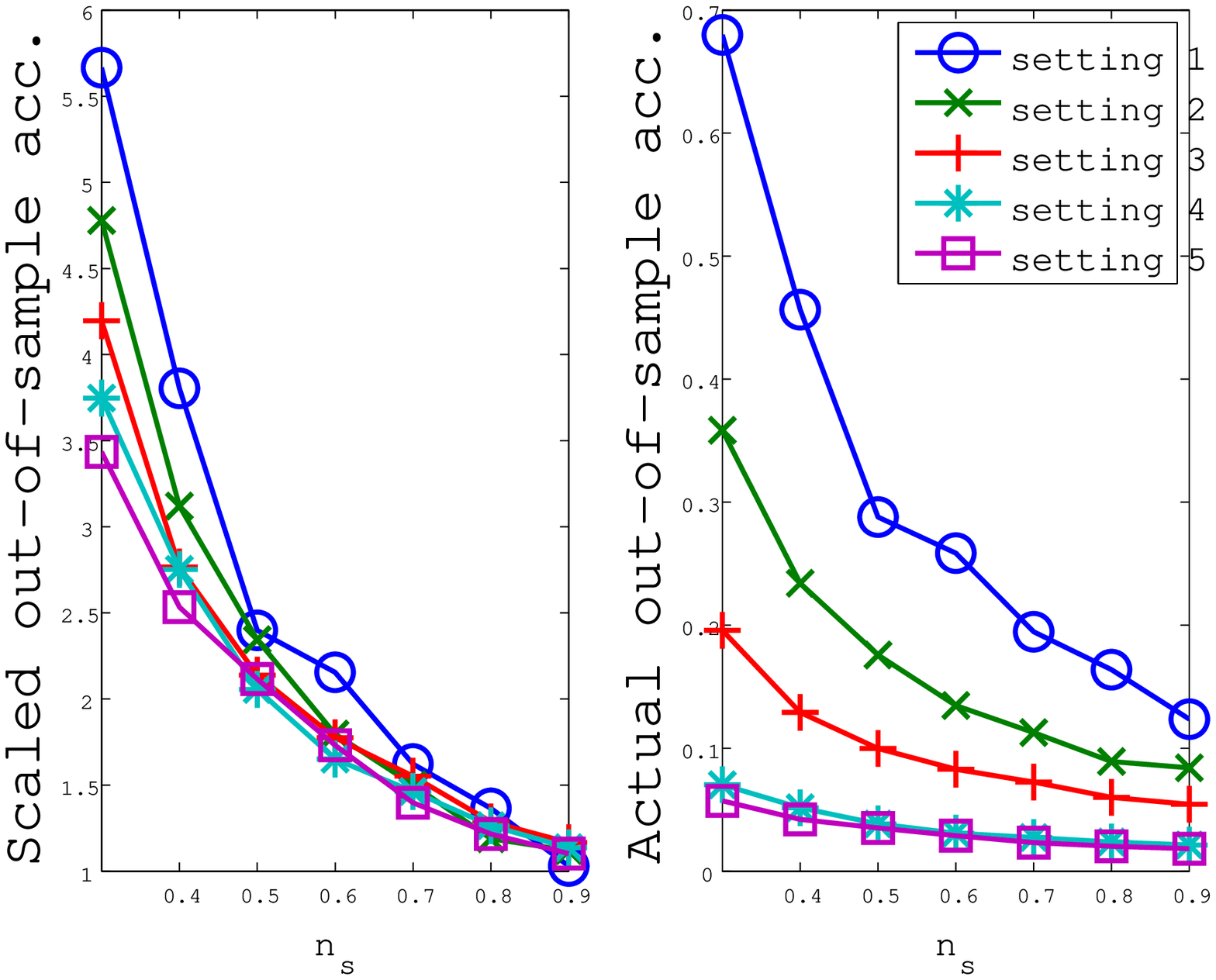}
\caption{Scaled out-of-sample accuracy (left) and actual out-of-sample accuracy
(right) versus $n_s$, averaged over five repetitions.}
\label{fig:outsample}
\end{center}
\end{minipage}
\end{figure}


\section{Conclusion and discussions}

In this paper, we investigated the statistical convergence rate of 
a Bayesian low rank tensor estimator.
The notion of a tensor's rank in this paper was based on CP-rank.
It is noteworthy that the predictive accuracy was bounded
{\it without} any strong convexity assumption.
Moreover, the obtained bound was (near) optimal
and automatically adapted to the unknown rank.
Numerical experiments showed that our theories indeed 
describe the actual behavior of the Bayes estimator.

Our bound includes the $\log$ term, which is not negligible when $K$ is large.
However, numerical experiments showed that the scaling factor without the $\log$ term explains well the actual behavior.
An investigation of whether the $\log$ term can be removed or not is 
an important future work. 


\section*{Acknowledgment}
We would like to thank Ryota Tomioka and Pierre Alquier for suggestive discussions.
This work was partially supported by MEXT Kakenhi 25730013 and JST-CREST.

\appendix 

\section{Proofs of Lemma 1 and Theorem 1}

Let $M_{r,d}$ be a positive number that will be determined later on.
Let $\calF_{r,d}$ be 
$$
\calF_{r,d} = \{A = [[U^{(1)},\dots,U^{(K)}]] \mid 
\max_{i,k} \|U_{:,i}^{(k)}\|_2 \leq \sqrt{M_{r,d}},~U^{(k)} \in \Real^{d \times M_{k}} \},
$$
and $\calF_{r}$ be $\cup_{d=1}^M \calF_{r,d}$.
$$
\|A^*\|_{\max,2} := \min_{\{U^{(k)}\}}\{ \max_{i,k} \|U^{(k)}_{:,i}\| \mid A^* = [[U^{(1)},\dots,U^{(K)}]],~U^{(k)} \in \Real^{d \times M_{k}} \}.
$$
$\sigmap{d^*}$ is denoted by $\sigmapo/\sqrt{d^*}$.

\subsection{Proof of Lemma 1}

For $d^*$ dimensional vectors $u^{(k)}~(k=1,\dots,K)$, let
$$
\langle u^{(1)},\dots,u^{(K)} \rangle := \sum_{i=1}^{d'} \prod_{k=1}^K u_i^{(k)}.
$$
Then for $u^{(k)},~u^{*(k)}~(k=1,\dots,K)$, we have that
\begin{align}
& |\langle u^{(1)},\dots,u^{(K)} \rangle - \langle u^{*(1)},\dots,u^{*(K)} \rangle| \notag \\
=
&
|\sum_{k=1}^{K} \langle u^{*(1)},\dots, u^{*(k-1)}, u^{(k)} - u^{*(k)}, u^{(k+1)}, \dots,u^{(K)} \rangle|  \notag \\
\leq 
&
\sum_{k=1}^K \|u^{*(1)}\| \cdots \|u^{*(k-1)}\|\|u^{(k)} - u^{*(k)}\| \|u^{(k+1)}\| \cdots \| u^{(K)} \|.
\label{eq:ukbounds}
\end{align}
Therefore, if $\|u^{(k)}\| \leq \|A^*\|_{\max,2} +\sigmapo$, 
$\|u^{(k)} - u^{*(k)}\| \leq \frac{\epsilon r}{K(\|A^*\|_{\max,2} +\sigmapo)^{K-1}}$ and 
$\|u^{*(k)}\| \leq \|A^*\|_{\max,2}$, 
then we have 
$$
|\langle u^{(1)},\dots,u^{(K)} \rangle - \langle u^{*(1)},\dots,u^{*(K)} \rangle|
\leq \epsilon r.
$$
Now, we consider two situations: 
(i) $\frac{\epsilon r}{K(\|A^*\|_{\max,2} +\sigmapo)^{K-1}} \leq \sqrt{d^*}\sigmap{d^*}$
and
(ii) $\frac{\epsilon r}{K(\|A^*\|_{\max,2} +\sigmapo)^{K-1}} \geq \sqrt{d^*}\sigmap{d^*}$.
(i) If $\frac{\epsilon r}{K(\|A^*\|_{\max,2} +\sigmapo)^{K-1}} \leq \sqrt{d^*}\sigmap{d^*}
(=\sigmapo)$ 
and $\|u^{*(k)}\| \leq \|A^*\|_{\max,2}$, then $\|u^{(k)}\| \leq \|A^*\|_{\max,2} +\sigmapo$.
Hence, by Lemma \ref{lemm:SmallBallProb}, we have that 
\begin{align*}
&-\log(\Pi(A:\|A-A^*\|_n < \epsilon r | d^*))
\leq 
-\log(\Pi(A:\|A-A^*\|_\infty < \epsilon r | d^*)) \\
\leq 
&
- \sum_{i_1,\dots,i_K}^{M_1,\dots,M_K} 
\log\left(\Pi ( U_{:,i_k}^{(k)} : \| U_{:,i_k}^{(k)} -  U_{:,i_k}^{*(k)}\| \leq 
{\textstyle \frac{\epsilon r}{K(\|A^*\|_{\max,2} +\sigmapo)^{K-1}}} )\right)   \\
\leq 
&
\sum_{k=1}^K \sum_{i_k=1}^{M_k}
\left\{
- d^* \log\left[ \frac{\epsilon r}{6\sigma_{\mathrm{p},d^*} \sqrt{d^*}
K(\|A^*\|_{\max,2} +\sigmapo)^{K-1}} \right] 
+ \frac{\|U_{:,i_k}^{*(k)}\|^2 }{2\sigma_{\mathrm{p},d^*}^2} \right\} \\
\leq 
&
- d^* (\sum_{k=1}^K M_k) \log\left[ \frac{\epsilon r}{6\sigma_{\mathrm{p},d^*} 
\sqrt{d^*} K(\|A^*\|_{\max,2} +\sigmapo)^{K-1}} \right] 
+ \frac{\sum_{k=1}^K  \|U^{*(k)}\|_{F}^2 }{2\sigma_{\mathrm{p},d^*}^2}.
\end{align*}
(ii) On the other hand,  if $\frac{\epsilon r}{K(\|A^*\|_{\max,2} +\sigmapo)^{K-1}} > 
\sigmap{d^*} \sqrt{d^*}$, 
we have that 
\begin{align*}
&-\log(\Pi(A:\|A-A^*\|_n < \epsilon r | d^*))
\leq 
- d^* (\sum_{k=1}^K M_k) \log\left( \frac{\sigmap{d^*}\sqrt{d^*}}{6\sigma_{\mathrm{p},d^*} \sqrt{d^*}} \right) 
+ \frac{\sum_{k=1}^K  \|U^{*(k)}\|_{F}^2 }{2\sigma_{\mathrm{p},d^*}^2}.
\end{align*}
Combining these inequality, we have that
\begin{align}
&-\log(\Pi(A:\|A-A^*\|_n < \epsilon r)) 
\leq 
-\log(\Pi(A:\|A-A^*\|_n < \epsilon r | d^*)) - \log(\pi(d^*)) \notag \\
\leq 
&
- d^* \left(\sum_{k=1}^K M_k \right)  \log\left[ \frac{1}{6} 
\left( {\textstyle \frac{\epsilon r}{\sigmapo K(\|A^*\|_{\max,2} +\sigmapo)^{K-1}}} \wedge 1\right) 
\right] 
+ \frac{\sum_{k=1}^K  \|U^{*(k)}\|_{F}^2 }{2\sigma_{\mathrm{p},d^*}^2} 
- \log(\pi(d^*)) \notag \\
\leq 
&
d^* \left(\sum_{k=1}^K M_k \right)  \log\left[ \frac{6}{\xi}
\left( {\textstyle \frac{\sigmapo^K K(\frac{\|A^*\|_{\max,2}}{\sigmapo}+1)^{K-1}}{\epsilon r}} \vee 1\right) \right] 
+ \frac{\sum_{k=1}^K  \|U^{*(k)}\|_{F}^2 }{2\sigma_{\mathrm{p},d^*}^2}. 
\label{eq:XiUpperBound1}
\end{align}
The last line is by the definition of 
the unnormalized prior $\pi(d)=\xi^{d(\sum_k M_k)}$.
Later on we let $\epsilon = 1/\sqrt{n}$.
This gives Lemma 1.


\subsection{Proof of Theorem 1}
\label{sec:ProofTh1}

In this section, we fix $X_{1:n}$ and think it is not a random variable.

To prove the theorem, we utilize the technique developed by \cite{JMLR:Vaart&Zanten:2011}.
Their technique is originally developed to show 
the posterior convergence of Gaussian process regression 
and is based on 
theories by \cite{AS:Ghosal+Ghosh+Vaart:2000} for the posterior convergence of 
non-parametric Bayes models.
Although our situation is of parametric model,
their technique is useful because ours is high dimensional singular model
in which  
a standard asymptotic statistics for parametric models does not work.

For a set of tensors $\calF_r$, an event $\calA_r$ and a test $\phi_r$
(all of which are dependent on a positive real number $r > 0$), it holds that,
for $\epsilon > 0$,
\begin{align}
 &\EE\left[\int \|A-A^*\|_n^2 \Pi (\dd A | Y_{1:n} ) \right] \notag \\
= &
\EE\left[32 \epsilon^2 \int_{r >0} r \Pi(\|A - A^*\|_n \geq 4 \epsilon r  | Y_{1:n} ) \dd r \right] \notag \\
\leq
&
32\epsilon^2 \int_{r >0} r^K 
\big\{ \EE\left [ \phi_r \right] 
+ 
P(\calA_r^c) \notag \\
&~~~~~~~~~~~~
+ 
\EE[(1-\phi_r) \boldone_{\calA_r} \Pi(A \in \calF_r^c | Y_{1:n}) ]
\notag \\
&
~~~~~~~~~~~~+ 
\EE[(1-\phi_r) \boldone_{\calA_r} \Pi(A \in \calF_r: \|A-A^*\|_n^2 \geq 4 \epsilon r^{2K}| Y_{1:n}) ] 
\big\} \dd r^K \notag \\
=:&
32\epsilon^2 \int_{r >0} r^K 
(A_r + B_r + C_r + D_r) \dd r^K.
\label{eq:decompIntABCD}
\end{align}
We give an upper bound of $A_r$, $B_r$, $C_r$ and $D_r$ in the following.

{\it Step 1.}
The probability distribution of $Y_{1:n}$ with a true tensor $A$ (that means $Y_i = \langle X_i, A\rangle + \epsilon_i$)
is denoted by $P_{n,A}$.
The expectation of a function $f$ with respect to $P_{n,A}$ is denoted by $P_{n,A}( f)$.

For arbitrary $r' >0$, define $C_{j,r',d} = \{ A \in \calF_{r,d} \mid jr' \leq 
\sqrt{n} \|A-A^*\|_n \leq (j+1)r' \}$. 
We construct a maximum cardinality set $\Theta_{j,r',d} \subset C_{j,r',d}$ such that 
each $A,A'\in \Theta_{j,r',d}$ satisfies $\sqrt{n}\|A-A'\|_n \geq jr'/2$.
The cardinality of $\Theta_{j,r',d}$ is equal to $D(jr'/2,C_{j,r',d},\sqrt{n}\|\cdot\|_n)$
\footnote{For a normed space $\calF$ attached with a norm $\|\cdot\|$,
the $\epsilon$-packing number is denoted by $D(\epsilon,\calF,\|\cdot\|)$.}.
Then, one can construct a test $\phi_{j,d}$ such that 
\begin{align*}
& P_{n,A^*} \phi_{j,d} \leq D(jr'/2,C_{j,r,d},\sqrt{n}\|\cdot\|_n) e^{-\frac{1}{2}(\frac{jr'}{2} + q)^2}
\leq D(jr'/2,C_{j,r,d},\sqrt{n}\|\cdot\|_n) e^{-\frac{1}{8}j^2r'^2 - \frac{1}{2}q^2}, \\
& \sup_{A\in C_{j,r',d}} P_{n,A}(1-\phi_{j,d}) \leq e^{-\frac{1}{2}\max(\frac{jr'}{2} - q,0)^2}
\leq e^{-\frac{j^2r'^2}{16} + \frac{q^2}{2}},
\end{align*}
for any $q > 0$~(see \cite{JMLR:Vaart&Zanten:2011} for the details ). 
For each $d$, we construct a test $\phi_d$ as $\phi_d = \max_j \phi_{j,d}$.
Then we have 
\begin{align*}
 P_{n,A^*} \phi_d 
&\leq \sum_{j \geq 1} D(jr'/2,C_{j,r',d},\sqrt{n}\|\cdot\|_n) e^{-\frac{1}{8}j^2r'^2- \frac{1}{2}q^2} 
\leq  \sum_{j \geq 1} D(r'/2,\calF_{r,d},\sqrt{n}\|\cdot\|_n) e^{-\frac{1}{8}j^2r'^2- \frac{1}{2}q^2} \\
&\leq 9 D(r'/2,\calF_{r,d},\sqrt{n}\|\cdot\|_n) e^{-\frac{1}{8}r'^2 - \frac{1}{2}q^2} .
\end{align*}
Here, by setting  $\frac{1}{2}q^2 = \log(D(r'/2,\calF_{r,d},\sqrt{n}\|\cdot\|_n))$, we have 
\begin{align*}
P_{n,A^*} \phi_d & \leq 9 e^{-\frac{1}{8}r'^2 }, \\
\sup_{A\in \calF_{r,d}} P_{n,A}(1-\phi_{d}) & \leq e^{-\frac{1}{16}j^2r'^2 + 
\log(D(r'/2,\calF_{r,d},\sqrt{n}\|\cdot\|_n))}.
\end{align*}

Finally, we construct a test $\phi$ as the maximum of $\phi_{d}$, that is, $\phi = \max_{d \geq 1} \phi_{d}$. Then,
we have
\begin{align*}
P_{n,A^*} \phi &
\leq 9 e^{-\frac{1}{8}r'^2  + \log(\dmax)} \\
\sup_{A\in \calF_{r,d}} P_{n,A}(1-\phi) & 
\leq e^{-\frac{1}{16}r'^2 + \log(D(r'/2,\calF_{r,d},\sqrt{n}\|\cdot\|_n)) },
\end{align*}
for all $d \geq 1$.

Substituting $4 \sqrt{n} \epsilon r^K$ into $r'$,
we obtain 
\begin{align}
P_{n,A^*} \phi &
\leq 9 e^{-2 n \epsilon^2 r^{2K} + \log(\dmax)} \\
\sup_{A\in \calF_{r,d}} P_{n,A}(1-\phi) & 
\leq e^{-n \epsilon^2 r^{2K} + \log(D(r'/2,\calF_{r,d},\sqrt{n}\|\cdot\|_n)) }.
\label{eq:testOuterBound}
\end{align}

We define 
$$
A_r = P_{n,A^*} \phi.
$$
From now on, we denote by $\phi_r$ the test constructed above to indicate that the test is associated to 
a specific $r$.

{\it Step 2.}

By Lemma 14 of \cite{JMLR:Vaart&Zanten:2011} and its proof, one can show that, for any $r > 0$, 
$$
P\left[
\int \frac{p_{n,A}}{p_{n,A^*}} \dd \Pi(A) \geq e^{-\frac{n\epsilon^2 r^2}{2}-
\sqrt{n}\epsilon rx}\Pi(A :\|A-A^*\|_n < \epsilon r)
\right]
\leq e^{-\frac{x^2}{2}}.
$$
Therefore, there exists an even $\mathcal{A}_{1,r}$ such that  
$$
P_{A^*}(\mathcal{A}_{1,r}^c) \leq e^{-n\epsilon^2 r^{2}/8},
$$
and, on the event $\calA_{1,r}$, it holds that
$$
\int \frac{p_{n,A}}{p_{n,A^*}} \dd \Pi(A) \geq e^{-n\epsilon^2 r^2} 
\Pi(A:\|A-A^*\|_n < \epsilon r).
$$

Moreover, it can be checked in a similar way that there exists an event $\calA_{2,r}$ such that, 
for a some fixed constant $\cepsilon < \frac{1}{2}$ (which will be determined later),
\begin{align}
P_{A^*}(\mathcal{A}_{2,r}^c) \leq \exp\left\{-\frac{1}{2}\left(
\frac{\sqrt{\cepsilon} \sqrt{n} \epsilon r^{K}}{2} 
+ \frac{\log(\Pi(A:\|A-A^*\|_n < \sqrt{\cepsilon}\epsilon r^K))}{\sqrt{\cepsilon} \sqrt{n} \epsilon r^{K}}
\vee 0 \right)^2\right\},
\label{eq:A2prob}
\end{align}
and, on the event $\calA_{2,r}$, it holds that
$$
\int \frac{p_{n,A}}{p_{n,A^*}} \dd \Pi(A) \geq e^{-\cepsilon n\epsilon^2 r^{2K}}.
$$
Here, we note that, by a simple calculation, the RHS of \Eqref{eq:A2prob} is bounded by 
$$
e^{-\frac{1}{16}\cepsilon n \epsilon^2 r^{2K}},
$$
if  $\cepsilon n \epsilon^2 r^{2K} \geq - 8\log(\Pi(A:\|A-A^*\|_n < \sqrt{\cepsilon} \epsilon r^{K}))$.

Now, define $\calA_r = \calA_{1,r} \cap \calA_{2,r}$, then we have 
$$
B_r = P_{A^*}(\calA_r^c) \leq e^{-n\epsilon^2 r^{2}/8} + 
\exp\left\{-\frac{1}{2}\left(
\frac{\sqrt{\cepsilon} \sqrt{n} \epsilon r^{K}}{2} 
+ \frac{\log(\Pi(A:\|A-A^*\|_n < \sqrt{\cepsilon}\epsilon r^K))}{\sqrt{\cepsilon} \sqrt{n} \epsilon r^{K}}
\vee 0 \right)^2\right\}.
$$

{\it Step 3.}

Since 
\begin{align*}
&\{\{U^{(k)}\}_{k=1}^K \in \Real^{d \times M_1} \times \dots \times  \Real^{d \times M_K}
\mid  
\max_{k,i}\|U_{:,i}^{(k)}\|_2 \geq \sqrt{M_{r,d}} \} \\
\subseteq 
&\{\{U^{(k)}\}_{k=1}^K \in \Real^{d \times M_1} \times \dots \times  \Real^{d \times M_K}
\mid  
\sum_{k=1}^K \|U^{(k)}\|_{2}^2 \geq M_{r,d} \},
\end{align*}
Proposition~\ref{lemm:TailBoundChi2} yields the bound of the prior probability of $\calF_{r}^c$ as 
\begin{align*}
\Pi(\calF_{r}^c) = \sum_{d=1}^M \Pi(\calF_{r,d}^c) \pi(d) \leq \sum_{d=1}^M 
\exp\left[\frac{d\sum_k M_k}{2} + \frac{d\sum_k M_k}{2}\log\left(\frac{M_{r,d}}{d \sum_k  M_k \sigmap{d}^2}\right) - \frac{M_{r,d}}{2\sigmap{d}^2}  \right] \pi(d).
\end{align*}
Therefore, its posterior probability in the event $\calA_r$ is bounded as 
\begin{align*}
C_{r} 
&=P_{n,A^*}[ \Pi(\calF_{r}^c|Y_{1:n})\boldone_{\calA_r}(1-\phi_r)] \\
&\leq
\sum_{d=1}^M 
\exp\left[n \epsilon^2 r^2 + \Xi(\epsilon r) + \frac{d\sum_k M_k}{2} + \frac{d\sum_k M_k}{2}
\log\left(\frac{M_{r,d}}{d \sum_k  M_k \sigmap{d}^2}\right) - \frac{M_{r,d}}{2\sigmap{d}^2} \right] \pi(d).
\end{align*}

{\it Step 4.}

Here, $D_r$ is evaluated.
Remind that $D_r$ is defined as  
$$
D_r = P_{n,A^*}[\Pi(A \in \calF_{r} : \|A - A^*\|_n > 4 \epsilon r^{K} | Y_{1:n})(1-\phi_r)\boldone_{\calA_r}].
$$
Since $\calA_r = \calA_{1,r} \cup \calA_{2,r} \supseteq \calA_{2,r}$, we have 
\begin{align*}
D_r 
&\leq   P_{n,A^*}\left[ \sum_{d} \int_{A \in \calF_{r,d} : \|A - A^*\|_n > 4 \epsilon r^{K}}  p_{n,A}/p_{n,A^*}  \dd  \Pi(A|d) 
\exp(\cepsilon n \epsilon^2 r^{2K}) 
\pi(d) (1-\phi_r)\boldone_{\calA_r} \right] \\
&=   \sum_d \int_{A \in \calF_{r,d} : \|A - A^*\|_n > 4 \epsilon r^{K}}  P_{n,A}[(1-\phi_r)\boldone_{\calA_r}] \exp(\cepsilon n \epsilon^2 r^{2K}) 
 \dd  \Pi(A|d) \pi(d).
\end{align*}
Therefore, using $P_{n,A}[(1-\phi_r)\boldone_{\calA_r}] \leq 1$, the summand in the RHS is bounded by 
\begin{align}
\pi(d)\exp(\cepsilon n \epsilon^2 r^{2K}).
\label{eq:DrSummandFirstBound}
\end{align}
Simultaneously, \Eqref{eq:testOuterBound} gives another upper bound of the summand as 
\begin{align}
& \int_{A \in \calF_{r,d} : \|A - A^*\|_n > 4 \epsilon r^{K}}  
e^{\cepsilon n \epsilon^2 r^{2K} - n \epsilon^2 r^{2K} + \log(D(2\sqrt{n}\epsilon r^K,\calF_{r,d},\sqrt{n}\|\cdot\|_n)} 
 \dd  \Pi(A|d) \pi(d)   \notag \\
\leq 
& 
 \pi(d) \exp\left\{ - \frac{1}{2}n\epsilon^2 r^{2K} + \log(D(2\sqrt{n}\epsilon r^K,\calF_{r,d}),\sqrt{n}\|\cdot\|_n))\right\},
\label{eq:DrSummandSecondBound}
\end{align}
for $r\geq 1$, where we used $\cepsilon < \frac{1}{2}$.

We evaluate the packing number $\log(D(2\sqrt{n}\epsilon r^K ,\calF_{r,d},\sqrt{n}\|\cdot\|_n))$.
It is known that the packing number of unit ball in $d$-dimensional Euclidean space is bounded by
$$
D(\epsilon, B_d(1), \|\cdot\|) \leq \left( \frac{4+\epsilon}{\epsilon} \right)^d.
$$
Here $B_d(R)$ denotes the ball with the radius $R$ in $d$-dimensional Euclidean space.

Similar to \Eqref{eq:ukbounds}, the $L_2$-norm between two tensors $A=[[U^{(1)},\dots,U^{(K)}]]$
and $A'=[[U'^{(1)},\dots,U'^{(K)}]]$
can be bounded by 
\begin{align}
\|A - A'\|_2^2 = 
& \sum_{i_1,i_2,\dots,i_K}
(\langle U_{:,i_1}^{(1)},\dots,U_{:,i_K}^{(K)} \rangle - \langle U_{:,i_1}'^{(1)},\dots,U_{:,i_K}'^{(K)} \rangle)^2 \notag \\
=
&
\sum_{i_1,i_2,\dots,i_K}
\left(\sum_{k=1}^K \langle U_{:,i_1}^{(1)},\dots,U_{:,i_{k-1}}^{(k-1)}, U_{:,i_{k}}^{(k)} -  U_{:,i_k}'^{(k)},
\dots,U_{:,i_K}'^{(K)} \rangle\right)^2 \notag \\
\leq
&
K
\sum_{i_1,i_2,\dots,i_K}
\sum_{k=1}^K 
\langle U_{:,i_1}^{(1)},\dots,U_{:,i_{k-1}}^{(k-1)}, U_{:,i_{k}}^{(k)} -  U_{:,i_k}'^{(k)},
U_{:,i_{k+1}}'^{(k+1)}, \dots,U_{:,i_K}'^{(K)} \rangle^2 \notag \\
\leq 
&
K
\sum_{k=1}^K
\sum_{i_1,i_2,\dots,i_K}
\| U_{:,i_1}^{(1)} \|_2^2 \times \cdots \times \|U_{:,i_{k-1}}^{(k-1)}\|_2^2 
\times
\|U_{:,i_{k}}^{(k)} -  U_{:,i_k}'^{(k)}\|_2^2
\times
\|U_{:,i_{k+1}}'^{(k+1)}\|_2^2
\times
\cdots 
\times
\|U_{:,i_K}'^{(K)}\|_2^2 \notag \\
\leq
&
K
\sum_{k=1}^K
\| U^{(1)} \|_2^2 \times \cdots \times \|U^{(k-1)}\|_2^2 
\|U^{(k)} -  U'^{(k)}\|_2^2 
\times
\|U'^{(k+1)}\|_2^2
\times
\cdots 
\times
\|U'^{(K)}\|_2^2.
\notag 
\end{align}
If $A,A' \in \calF_{r,d}$, then the RHS is further bounded by 
\begin{equation}
\|A - A'\|_2^2 \leq 
K M_{r,d}^{K-1} \sum_{k=1}^K \|U^{(k)} -  U'^{(k)}\|_2^2.
\label{eq:AL2bound}
\end{equation}
Thus, if $2\epsilon r^K \leq KM_{r,d}^{K/2}$, using the relation \eqref{eq:AL2bound}
and $\sqrt{n}\|A-A'\|_n \leq \sqrt{n}\|A-A'\|_\infty \leq \sqrt{n}\|A-A'\|_2$, 
we have that 
\begin{align}
&\log(D(2\sqrt{n}\epsilon r^K,\calF_{r,d},\sqrt{n}\|\cdot\|_n)) 
\leq
\log(D(2 \epsilon r^K,\calF_{r,d}, \|\cdot\|_2))  \notag \\ 
&
\leq
\log(D(2\epsilon r^K/(KM_{r,d}^{(K-1)/2}),B_{d\sum_{k} M_k}(\sqrt{M_{r,d}}),\|\cdot\|_2)) 
\notag \\
&\leq 
d\left(\sum_{k=1}^K M_k\right) \log\left(\frac{4+\frac{2\epsilon r^K}{KM_{r,d}^{K/2}}}{\frac{2\epsilon r^K}{KM_{r,d}^{K/2}}}\right) 
\leq 
d(\sum_{k=1}^K M_k) \log\left(\frac{3KM_{r,d}^{K/2}}{\epsilon r^K}\right),
\label{eq:DrCoveringBound}
\end{align}
otherwise, $\log(D(2\sqrt{n}\epsilon r^K,\calF_{r,d})) = 0$.

Combining \Eqrefs{eq:DrSummandFirstBound} and \eqref{eq:DrSummandSecondBound} with \Eqref{eq:DrCoveringBound}  
results in the following upper bound of $D_r$:
\begin{align}
&D_r 
\leq  
\sum_d  \pi(d)
\min\left\{
\exp(\cepsilon n \epsilon^2 r^{2K}),
 \exp\left[ - \frac{1}{2}n\epsilon^2 r^{2K} + d\left(\sum_{k=1}^K M_k\right) \log\left(\frac{3KM_{r,d}^{K/2}}{\epsilon r^K}\right)\right]
\right\}
\label{eq:DrFinalBound}
\end{align}
for $r \geq 1$.

{\it Step 5.}

Here, we establish the assertion by combining the bounds of $A_r,~B_r,~C_r$ and $D_r$ obtained above.
Set $\epsilon = \frac{1}{\sqrt{n}}$ and $M_{r,d} = \frac{4}{\min(d^*,d)} \Xi(1/\sqrt{n})\sigmapo^2 r^2$. 
Then, we have that, for all $d \geq d^*$, 
$$
d\left( \sum_{k=1}^K M_k \right) \log\left(\frac{3KM_{r,d}^{K/2}}{\epsilon r^K} \right)
\leq
d\left( \sum_{k=1}^K M_k \right) \log\left(3K \sqrt{n} \left(\frac{4\sigmapo^2 \Xi(\frac{1}{\sqrt{n}})}{d^*}
\right)^{\frac{K}{2}} \right).
$$
Recall that $C_{n,K} = 3K \sqrt{n} \left(\frac{4 \sigmapo^2 \Xi(\frac{1}{\sqrt{n}})}{d^*} \right)^{\frac{K}{2}} $,
and $\cepsilon = \min\{ |\log(\xi)|/\log(C_{n,K}),1 \}/4$.
Let $\tilde{r}_d$ be such that 
$-\frac{1}{2}\tilde{r}_d^{2K} + d\left(\sum_{k=1}^K M_k\right)\log(C_{n,K})  = 0$, 
that is, 
$\tilde{r}_d = 2^{1/2K}[d(\sum_k M_k) \log(C_{n,K})]^{1/2K}$.

By the upper bound \eqref{eq:decompIntABCD}, we have that 
\begin{align*}
&\EE\left[\int \|A-A^*\|_n^2 \dd\Pi(A|Y_{1:n}) \right] 
\leq 
32\epsilon^2 + 32\epsilon^2
\int_{r>1} r^K(A_r + B_r + C_r + D_r) \dd r^K.
\end{align*}

We are going to bound each term in the integral.
\begin{align*}
\int_{r>1} r^K A_r \dd r^K 
\leq 
\int_{r>1}
r^K 
\min\{9\exp(-2r^{2K} + \log(\dmax)),1\} \dd r^K \leq \frac{\log(\dmax)}{2} + \frac{9}{4}.
\end{align*}

By Lemma \ref{lemm:acxKintegral}, the integral related to $B_r$ is bounded by
\begin{align*}
\int_{r>1} r^K B_r \dd r^K 
&\leq 
8\Xi(\sqrt{\cepsilon} \epsilon)/\cepsilon
+
\int_{r^K>\sqrt{8\Xi(\sqrt{\cepsilon} \epsilon)/\cepsilon}}
r^K
\left[
\exp(- r^{2}/8) 
+
\exp(-\cepsilon r^{2K}/16)\right] \dd r^K  \\
&\leq 
8\frac{\Xi(\sqrt{\cepsilon} \epsilon)}{\cepsilon} + 
\int_{r^K > 1} r^K \exp(-r^2/8) \dd r^K
\\
&~~~~+
\frac{1}{32} \exp\left(-\frac{\Xi(\sqrt{\cepsilon} \epsilon)}{2\cepsilon}\right).
\end{align*}
Here, the second term of the RHS is bounded by 
\begin{align*}
&\int_{r^K > 1} r^K \exp(-r^2/8) \dd r^K
\leq \frac{1}{2} \sum_{j=1}^K 8^j \frac{K!}{(K-j)!} \exp\left(-\frac{1}{8}\right) 
\leq 8^K (K+1)!.
\end{align*}

Similarly, by Lemma \ref{lemm:acxKintegral}, the integral related to $C_r$ is bounded by
\begin{align*}
\int_{r>1} r^K C_r \dd r^K 
&
\leq 
\int_{r>1}
r^K
\sum_{d=1}^M \exp\left[ -\frac{d}{\min(d,d^*)} \Xi(\epsilon)r^2 \right] \pi(d) \dd r^K  \\
&
\leq
\frac{1}{2}\sum_{j=0}^{K-1} \frac{K \cdots (K-j)}{\Xi(1/\sqrt{n})^{j+1}} \leq K.
\end{align*}

Next, we bound the integral corresponding to $D_r$.
By the definition of $\cepsilon$ and $\tilde{r}_d$, 
for all $r \leq \tilde{r}_d$, it holds that
$$
\log[\pi(d)\exp(\cepsilon r^{2K})]
\leq d'(\sum_k M_k) \log(\xi) + \cepsilon \tilde{r}_{d}^{2K} 
\leq 
d(\sum_k M_k) \log(\xi)/2,
$$
(remind that we are using unnormalized prior $\pi(d)$).
On the other hand, 
for all $r \geq \tilde{r}_d$,
it holds that 
$$
\exp\left\{ - n\epsilon^2 r^{2K}/2 + 
d\left(\sum_{k=1}^K M_k\right) \log\left(\frac{3KM_{r,d}^{K/2}}{\epsilon r}\right) \right\}
\leq 
\exp\left[ - (r^{2K} - \tilde{r}_d^{2K})/2 \right].
$$
Therefore, \Eqref{eq:DrFinalBound} becomes 
\begin{align}
& 
\int_{r>1} r^K D_r \dd r^K  
=
\int_{r=1}^{\tilde{r}_{d^*}} r^K D_r \dd r^K  
+ 
\int_{r > \tilde{r}_{d^*}} r^K D_r \dd r^K  
\notag \\
\leq
&
\frac{1}{2}\tilde{r}_{d^*}^{2K}
+
\sum_{d=1}^{d^*-1} 
\pi(d) \int_{\tilde{r}_{d^*} \leq r} r^K \exp(- (r^{2K} - \tilde{r}_{d}^{2K})) \dd r^K  \notag \\
&+
\sum_{d=d^*}^{M}
\Big\{
\int_{\tilde{r}_{d^*} \leq r \leq \tilde{r}_{d}} r^K  \dd r^K  \exp\left[ \frac{1}{2}d(\sum\nolimits_k M_k) \log(\xi)\right] 
+
\pi(d) \int_{\tilde{r}_{d} \leq r} r^K \exp\left[- (r^{2K} - \tilde{r}_{d}^{2K})\right] \dd r^K  
\Big\}
\label{eq:DrIntFirstBound}
\end{align}
Here the second term of the RHS can be evaluated as
\begin{align*}
&\sum_{d=1}^{d^*-1} 
\pi(d) \int_{\tilde{r}_{d^* \leq r}} r^K \exp(- (r^{2K} - \tilde{r}_{d}^{2K})) \dd r^K  
\leq 
\sum_{d=1}^{d^*-1} 
\pi(d).
\end{align*}
The third term is evaluated as 
\begin{align*}
&\sum_{d=d^*}^{M}
\int_{\tilde{r}_{d^*} \leq r \leq \tilde{r}_{d}} r^K  \dd r^K  \exp\left[ \frac{1}{2}d\left(\sum\nolimits_k M_k\right) \log(\xi)\right] \\
\leq
&
\sum_{d=d^*}^{M}
\frac{1}{2}\tilde{r}_{d}^{2K} \exp\left[ \frac{1}{2}d\left(\sum\nolimits_k M_k\right) \log(\xi)\right] \\
=
&
\sum_{d=d^*}^{M}
d(\sum\nolimits_k M_k)\log(C_{n,K}) \exp\left[ \frac{1}{2}d\left(\sum_k M_k\right) \log(\xi)\right] \\
\leq
&
\int_{d \geq d^*}  (d+1)(\sum\nolimits_k M_k)\log(C_{n,K}) 
\exp\left[ \frac{1}{2}d\left(\sum_k M_k\right) \log(\xi)\right] \dd d \\
\leq
&
\frac{2\sum_k M_k(d^*+2)}{\sum\nolimits_k M_k |\log(\xi)|}\log(C_{n,K})
\exp\left[\frac{d^*(\sum_k M_k)}{2}\log(\xi)\right] \\
\leq 
&
\frac{2(d^*+2)}{|\log(\xi)|}\log(C_{n,K})
\end{align*}
The fourth term is bounded in a similar way to the second term as
\begin{align*}
\sum_{d=d^*}^{M}
\pi(d) \int_{\tilde{r}_{d} \leq r} r^K \exp(- (r^{2K} - \tilde{r}_{d}^{2K})) \dd r^K  
\leq \sum_{d=d^*}^{M}\pi(d).
\end{align*}
Thus \Eqref{eq:DrIntFirstBound} is upper bounded by 
\begin{align*}
&\int_{r \geq 1} r^K D_r \dd r^K  \leq 
\frac{1}{2}\tilde{r}_{d^*}^{2K} + 2 +  
\frac{2(d^*+2)}{|\log(\xi)|}\log(C_{n,K}) \\
=&
d^* \left(\sum\nolimits_k M_k\right)\log(C_{n,K}) 
+ 2 +  
\frac{2(d^*+2)}{|\log(\xi)|}\log(C_{n,K}).
\end{align*}

Combining all inequalities
yields that there exists a universal constant $C$
such that
\begin{align}
&\EE\left[\int\|A-A^*\|_n^2 \dd \Pi(A|Y_{1:n})\right] \notag \\
\leq 
&
\frac{C}{n}
\Bigg(
\frac{\log(\dmax)}{2} + K + 
8\frac{\Xi(\sqrt{\frac{\cepsilon}{n}})}{\cepsilon} + 
8^K (K+1)!
+
d^*\left(\sum\nolimits_k M_k\right)\log(C_{n,K}) \notag \\
&+ 
\frac{2(d^*+2)}{|\log(\xi)|}\log(C_{n,K})
+ 1
\Bigg).
\label{eq:AAstarIntBound}
\end{align}

\section{Rejection Sampling (Proof of Theorem 2)}
\label{sec:ProofTh2}

In this section, we prove Theorem 2. 
By assumption, we have $R \geq 2 \|A^*\|_{\infty}$ and $\frac{R}{2} \geq 1/\sqrt{n}$.
Therefore, if $\epsilon r \geq 1/\sqrt{n}$, we have that 
\begin{align}
&-\log(\Pi(A: \|A-A^*\|_n \leq \epsilon r, \|A\|_{\infty} \leq R))
\leq \Xi(\frac{R}{2} \wedge \epsilon r)  \leq \Xi(1/\sqrt{n}) .
\label{eq:XiBound2}
\end{align}
Now, for any non-negative measurable  function $f:\Real^{M_1\times \dots \times M_K}$, we have  that 
\begin{align*}
\int f(A) \Pi(\dd A  | \|A\|_{\infty} \leq  R, D_n )
&=
\frac{\int f(A) \boldone[\|A\|_{\infty} \leq  R] \Pi(\dd A  |  D_n )}
{\Pi( \|A\|_{\infty} \leq  R | Y_{1:n})} \\
&=
\frac{\int f(A) \boldone[\|A\|_{\infty} \leq R] \Pi(\dd A  |  D_n )}
{\Pi( \|A\|_{\infty} \leq  R | Y_{1:n})} \\
&=
\frac{\int f(A) \boldone[\|A\|_{\infty} \leq  R] \frac{p_{n,A}}{p_{n,A^*}}\Pi(\dd A )}
{\int \boldone[\|A\|_{\infty} \leq  R]   \frac{p_{n,A}}{p_{n,A^*}}\Pi(\dd A ) }. 
\end{align*}
We define the event $\calA_r$ as in the proof of Theorem 1.
Then, we have the same upper bound of $B_r$ as in the previous section.
By \Eqref{eq:XiBound2}, on the event $\calA_r$,
$$
\int \boldone[\|A\|_{\infty} \leq  R]   \frac{p_{n,A}}{p_{n,A^*}}\Pi(\dd A ) 
\geq \min\{\exp(r^2 + \Xi(1/\sqrt{n})),\exp(\cepsilon r^{2K})\}^{-1}.
$$
Other quantities such as $A_r,C_r,D_r$ are also bounded in the same manner 
because 
$$
\int f(A) \boldone[\|A\|_{\infty} \leq  M] \frac{p_{n,A}}{p_{n,A^*}}\Pi(\dd A )
\leq 
\int f(A)  \frac{p_{n,A}}{p_{n,A^*}}\Pi(\dd A ).
$$
Thus, 
the conditional posterior mean of the squared error 
$\EE\left[\int \|A-A^*\|_n^2 \dd\Pi(A|\|A\|_{\infty} \leq  M, Y_{1:n}) \right]$
can be bounded by the same quantity as the RHS of \Eqref{eq:AAstarIntBound}.

Now, we are going to bound the out-of-sample predictive error:
\begin{align}
\EE_{D_n}\left[\int  \|A - A^* \|_{\LPi}^2 \dd\Pi(A|\|A\|_{\infty} \leq  R, D_n) \right]. 
\label{eq:ExOutPredError}
\end{align}
By the assumption that $\|X\|_1 \leq 1$ a.s., we have that 
$$
|\langle X, A - A^*\rangle| \leq 2 R.
$$
Now, we bound the expected error \eqref{eq:ExOutPredError} by 
$I + II + III$ where, for $\eta = \epsilon \max\{4R,1\}=\frac{1}{\sqrt{n}} \max\{4R,1\}$, $I$, $II$ and $III$ are defined by
\begin{align*}
&
I= \EE_{D_n}\left[\int_{0}^\infty \eta^2 r \boldone_{\calA_r^c} \dd r \right],  \\
& 
II = \EE_{D_n}\left[\int_{0}^\infty 
\eta^2r   \boldone_{\calA_r} \Pi(A: \sqrt{2} \|A - A^*\|_n >  \eta r \mid \|A\|_{\infty} \leq  R, D_n) \dd r \right] ,
 \\
&
III= \EE_{D_n}\left[\int_{0}^\infty \eta^2 r \boldone_{\calA_r}
\Pi(A: \|A - A^*\|_{\LPi}  > \eta r \geq \sqrt{2} \|A - A^*\|_n \mid  \|A\|_{\infty} \leq  R, D_n) 
\dd r
\right].
\end{align*}

Here, $I$ and $II$ are bounded by 
\begin{align}
&I + II \notag \\
\leq 
&
C \eta^2
\Bigg(
\log(\dmax) + K + 
\frac{\Xi(\sqrt{\frac{\cepsilon}{n}})}{\cepsilon} + 
8^K (K+1)!
+
d^*\left(\sum\nolimits_k M_k\right)\log(C_{n,K}) \notag \\
&+ 
\frac{d^*+2}{|\log(\xi)|}\log(C_{n,K})
+ C 
\Bigg),
\label{eq:AAstarIandIIbound}
\end{align}
as in \Eqref{eq:AAstarIntBound}.

Next, we bound the term $III$.
To bound this, we need to evaluate the difference between the empirical norm $\|A-A^*\|_n$ 
and the expected norm $\|A-A^*\|_{\LPi}$, which can be done by Bernstein's inequality:
$$
P\left(\|A-A^*\|_{\LPi} \geq \sqrt{2} \|A-A^*\|_n \right)
\leq \exp\left(- \frac{n \|A-A^*\|_{\LPi}^4}{2(v + \|A-A^*\|_{\infty}^2/3)}  \right),
$$
where $v = \EE_X[(\langle X, A- A^*\rangle^2 -\EE_{\tilde{X}}[\langle \tilde{X}, A- A^*\rangle^2])^2 ]$.
Now $v \leq \EE_X[\langle X, A- A^*\rangle^4] 
\leq \|A-A^*\|_\infty^2\EE_X[\langle X, A- A^*\rangle^2]
=\|A-A^*\|_\infty^2 \|A-A^*\|_{\LPi}^2$. This yields that
$$
P\left(\|A-A^*\|_{\LPi} \geq \sqrt{2} \|A-A^*\|_n \right)
\leq \exp\left[-\frac{n }{2 }  
\left( \frac{\|A-A^*\|_{\LPi}}{ \|A-A^*\|_{\infty}} \right)^2 \right].
$$
If $\|A - A^* \|_{\infty} \leq 2R$, then the RHS is further bounded by 
$
\exp\left(-\frac{n \|A-A^*\|_{\LPi}^2}{8R^2}  \right).
$

To evaluate $III$, we evaluate the expectation of the posterior inside the integral:
\begin{align*}
&
\EE_{D_n}\left[
\boldone_{\calA_r}
\Pi(A: \|A - A^*\|_{\LPi}  > \eta r \geq \sqrt{2} \|A - A^*\|_n \mid  \|A\|_{\infty} \leq  R, D_n) 
\right] \\
\leq
& 
\EE_{X_{1:n}}\left[
\exp(r^2 + \Xi(\epsilon))
\Pi(A: \|A - A^*\|_{\LPi}  > \eta r \geq \sqrt{2} \|A - A^*\|_n,  \|A\|_{\infty} \leq  R  \mid X_{1:n}) 
\right]  \\
\leq
& 
\exp( r^2 + \Xi(1/\sqrt{n}))
\int_A \boldone[\|A-A^*\| >  \eta r,\|A\|_\infty \leq R] 
P(\|A - A^*\|_{\LPi}  >  \eta r \geq \sqrt{2} \|A - A^*\|_n) \Pi(\dd A) \\
\leq
&
\exp( r^2 + \Xi(1/\sqrt{n}) )
\int_A \boldone[\|A-A^*\| >  \eta r] 
\exp\left(-\frac{n \|A-A^*\|_{\LPi}^2}{8R^2}  \right)
\Pi(\dd A)  \\
\leq 
&\exp\left(r^2 + \Xi(1/\sqrt{n}) - \frac{n\eta^2 r^2}{8R^2}\right).
\end{align*}
Therefore, we have that 
\begin{align}
III \leq & \int_{0}^\infty \eta^2 r 
\min\left\{ 1, \exp\left(r^2 + \Xi(1/\sqrt{n}) - \frac{n\eta^2 r^2}{8R^2}\right)\right\} \dd r \notag \\
\leq &
\frac{1}{2} \eta^2 \Xi(1/\sqrt{n}) + 
\int_{r > \sqrt{\Xi(1/\sqrt{n})}}^\infty \eta^2 r 
 \exp\left(-  r^2 \right) \dd r
\notag \\
\leq
&
C \eta^2 \Xi(1/\sqrt{n}),
\label{eq:IIIFinalbound}
\end{align}
where $\eta = \frac{1}{\sqrt{n}}\max\{4R,1\}$ is used in the second inequality.

Combining the inequalities \eqref{eq:AAstarIandIIbound} and \eqref{eq:IIIFinalbound},
we obtain that 
\begin{align*}
&
\EE_{D_n}\left[\int  \|A - A^* \|_{\LPi}^2 \dd\Pi(A|\|A\|_{\infty} \leq  R, D_n) \right] 
\notag \\
\leq
&
\frac{C\max\{R^2,1\}}{n}
\Bigg(
\frac{\log(\dmax)}{2} + 
8^K (K+1)! +
\frac{\Xi(\sqrt{\frac{\cepsilon}{n}})}{\cepsilon} 
+ \Xi(1/\sqrt{n})
\\
&
+
d^*\left(\sum\nolimits_k M_k\right)\log(C_{n,K}) 
+ 
\frac{d^*+2}{|\log(\xi)|}\log(C_{n,K})
+ C 
\Bigg),
\end{align*}
where $C$ is a universal constant.
Noticing $\frac{\Xi(\sqrt{\frac{\cepsilon}{n}})}{\cepsilon} \geq \Xi(1/\sqrt{n})$, we obtain the assertion.

\section{Rejection Sampling II (Proof of Theorem 3)}
\label{sec:ProofTh3}

In this section, the proof of Theorem 3 is given.
Let
\begin{align*}
\mathcal{U}_{R,d}&=
\{ (U^{(1)}, \dots, U^{(K)}) \in \Real^{d\times M_1} \times \dots \times \Real^{d\times M_K} 
\mid 
\|U^{(k)}_{:,j}\|_2 \leq R~(1 \leq k \leq \dmax, 1\leq j \leq M_k ).
\}, \\
\mathcal{U}_{R}&=\bigcup_d \mathcal{U}_{R,d}.
\end{align*}
Define
$$\calF_{r,d} = 
\left\{ A \in \Real^{M_1 \times \dots \times M_K} \mid 
\|A\|_{\max,2} \leq R \right\}.
$$

For $\boldU = (U^{(1)},\dots,U^{(K)})$, we denote by $A_{\boldU} := [[U^{(1)},\dots,U^{(K)}]]$.
Since the rejection sampling considered here is defined on a factorization $A=[[U^{(1)},\dots,U^{(K)}]]$,
the quantities $I,II,III$ are redefined as
\begin{align*}
&
I= \EE_{D_n}\left[\int_{0}^\infty \eta^2 r \boldone_{\calA_r^c} \dd r \right],  \\
& 
II = \EE_{D_n}\left[\int_{0}^\infty 
\eta^2 r   \boldone_{\calA_r} \Pi(\boldU: \sqrt{2} \|A_{\boldU} - A^*\|_n > \eta r \mid \boldU \in \calU_R, D_n) \dd r \right] ,
 \\
&
III= \EE_{D_n}\left[\int_{0}^\infty \eta^2 r \boldone_{\calA_r}
\Pi(\boldU : \|A_{\boldU}  - A^*\|_{\LPi}  > \eta r \geq \sqrt{2} \|A_{\boldU}  - A^*\|_n \mid  \boldU \in \calU_R, D_n) 
\dd r
\right].
\end{align*}

Since $R \geq \|A^*\|_{\max,2}+ \sigmapo$ is satisfied,
one can apply the same upper-bound evaluation of $\Pi(\boldU : \|A_{\boldU} - A^*\|_n < \epsilon,
A_{\boldU} \in \calU_R | d^*)$ as \Eqref{eq:XiUpperBound1}.

Then, in the same event $\calA_r$ as before, we have that
$$
\int \boldone[\|A\|_{\infty} \leq  R]   \frac{p_{n,A}}{p_{n,A^*}}\Pi(\dd A ) 
\geq \min\{\exp(r^2 + \Xi(1/\sqrt{n})),\exp(\cepsilon r^{2K})\}^{-1},
$$
by \Eqref{eq:XiBound2}.


$A_r,D_r$ can be bounded in a similar fashion, except we use $R$ instead of $M_{r,d}$
and $\mathcal{U}_R$ instead of $\calF_{r}$.
Now, in this case, $C_r$ could be zero. 
That is,
\begin{align*}
C_r = \EE_{Y_{1:n}}[\Pi(\mathcal{U}_R^c | \{U^{(k)}\}_k \in \mathcal{U}_R, Y_{1:n}) \boldone_{\calA_r} ]
=0.
\end{align*}

Thus, by resetting $M_{r,d} = R$ and, accordingly,  re-defining 
$$
C_{n,K} = 3K\sqrt{n} R^{\frac{K}{2}},
$$
then
we have the same bound of $I+II$ as \Eqref{eq:AAstarIandIIbound}
except the definition of $C_{n,K}$.

Because $\|A_{\boldU}\|_{\infty} \leq R^K$ for $\boldU \in \calU_R$, 
$III$ is bounded by $\frac{C}{n}R^{2K}\Xi(\epsilon)$.

Therefore, we again obtain that
\begin{align*}
& 
\EE_{D_n}\left[\int \|A_U - A^*\|_{\LPi}^2 \dd \Pi(A_U \mid U \in \calU_R, D_n)\right] \\
\leq 
&
\frac{C\max\{R^{2K},1\}}{n}
\Bigg(
\log(\dmax) + 
8^K (K+1)! +
\frac{\Xi(\sqrt{\frac{\cepsilon}{n}})}{\cepsilon} +
\Xi(\frac{1}{\sqrt{n}})  \\
&+ 
d^*\left(\sum\nolimits_k M_k\right)\log(K\sqrt{n}R^{\frac{K}{2}})
+ 
\frac{d^*+2}{|\log(\xi)|}\log(K\sqrt{n}R^{\frac{K}{2}})
+ C 
\Bigg).
\end{align*}
This yields the assertion.

\section{Auxiliary lemmas}
\begin{proposition}{Tail bound of $\chi^2$ distribution}
\label{lemm:TailBoundChi2}
Let $\chi_k$ be the chi-square random variable with freedom $k$.
Then the tail probability is bounded as 
$$
P(\chi_k \geq k + 2\sqrt{xk} + 2x) \leq \exp(-x) 
$$
$$
P(\chi_k \geq x) \leq \exp\left[\frac{k}{2} + \frac{k}{2}\log\left(\frac{x}{k}\right) - \frac{x}{2}\right],
$$
for $x \geq k$.
\end{proposition}
The proofs of the first assertion and the second one are 
given in Lemma 1 of \cite{AS:Laurent+Massart:2002}
and Lemma 2.2 of \cite{RSA:Dasgupta+Gupta:2002} respectively. 

\begin{Lemma}
\label{lemm:SmallBallProb}
The small ball probability of $d$-dimensional Gaussian random variable $g \sim N(0,\sigma \mathrm{I}_d)$ is lower bounded as
$$
P\left(\|g\| \leq \epsilon \right) \geq \left(\frac{\epsilon}{3\sigma d^{\frac{1}{2}}} \right)^d,
$$
for all $\epsilon \leq \sigma \sqrt{d}$.
\end{Lemma}
\begin{proof}
The lower bound is obtained in an almost same manner as 
the proof of Proposition 2.3 of \cite{AP:Li&Lindre:1999}.
If $\delta \leq 1$, then we have 
$$
P(|g_j| \leq \delta \sigma) = \frac{2}{\sqrt{2\pi \sigma^2}} \int_{0}^{\delta \sigma}
\frac{2\delta}{\sqrt{2\pi}} \exp\left(-\frac{x^2}{2\sigma^2}\right) \dd x
\geq \frac{2\delta}{\sqrt{2\pi}} \exp(-\frac{1}{2}) \geq \frac{\delta}{3}.
$$
Therefore, for $\epsilon$ such that $\epsilon \leq \sigma \sqrt{d}$, it holds that
\begin{align*}
P\left(\|g\| \leq \epsilon \right) 
&\geq 
P\left(\sqrt{d} \max_j |g_j| \leq \epsilon \right)  \geq
\prod_{j=1}^d P\left(\sqrt{d} |g_j| \leq \epsilon \right) \\
&\geq  \left( \frac{\epsilon}{3\sigma d^{\frac{1}{2}}}\right)^{d}.
\end{align*}
\end{proof}

\begin{Lemma}
\label{lemm:acxKintegral}
For all $a >0,c>0,K \geq 1$, it holds that 
$$
\int_{x \geq a} x \exp\big(-c x^{\frac{2}{K}}\big) \dd x 
= \frac{1}{2}\sum_{i=1}^{K} \frac{K\cdots (K-i+1)}{c^{i}}a^{(K-i)2/K} \exp(-ca^{2/K}).
$$
\end{Lemma}

\begin{proof}
\begin{align*}
&\int_{x \geq a} x \exp(-c x^{\frac{2}{K}}) \dd x \\
=&
\frac{K}{2} \int_{x \geq a^{2/K}} x^{K-1} \exp(-cx) \dd x \\
=&
\frac{K}{2}\left[-\frac{1}{c}x^{K-1}\exp(-cx) \right]_{a^{2/K}}^{\infty}
+ \frac{K(K-1)}{2c} \int_{a^{2/K}}^\infty x^{K-2}\exp(-cx) \dd x \\
=&
\frac{K}{2c}a^{(K-1)2/K} \exp(-ca^{2/K}) 
+ \frac{K(K-1)}{2c} \int_{a^{2/K}}^\infty x^{K-2}\exp(-cx) \dd x.
\end{align*}
Then, applying recursive argument we obtain the assertion.
\end{proof}

\begin{Lemma}
For all $K\geq 2$, we have
$$
\sum_{i=1}^d u^{(1)}_i u^{(2)}_i \cdots u^{(K)}_i
\leq 
\prod_{k=1}^K \|u^{(k)}\|_K. 
$$
In particular, 
$$
\sum_{i=1}^d u^{(1)}_i u^{(2)}_i \cdots u^{(K)}_i
\leq 
\prod_{k=1}^K \|u^{(k)}\|_2.
$$
\end{Lemma}
\begin{proof}
By H\"older's inequality, it holds that
\begin{align*}
\sum_{i=1}^d u^{(1)}_i u^{(2)}_i \cdots u^{(K)}_i
\leq \|u^{(1)}\|_{K}
\|(u^{(2)}_i \cdots u^{(K)}_i)_{i=1}^d\|_{K/(K-1)}.
\end{align*}
Applying H\"older's inequality again, we observe that
\begin{align*}
&~~~\|(u^{(2)}_i \cdots u^{(K)}_i)_{i=1}^d\|_{K/(K-1)} 
= \left(\sum_{i=1}^d |u^{(2)}_i \cdots u^{(K)}_i|^{K/(K-1)}\right)^{(K-1)/K} \\
&
\leq
\left[ \left(\sum_{i=1}^d |u^{(2)}|^K\right)^{1/K-1} 
\left(\sum_{i=1}^d |u^{(3)}_i \cdots u^{(K)}_i|^{K/(K-2)}\right)^{(K-2)/(K-1)} \right]^{(K-1)/K} 
\\
&
\leq
\left[ \left(\sum_{i=1}^d |u_i^{(2)}|^K\right)^{1/(K-1)} 
\left(\sum_{i=1}^d |u^{(3)}_i \cdots u^{(K)}_i|^{K/(K-2)}\right)^{(K-2)/(K-1)} \right]^{(K-1)/K}\\
&=
\|u^{(2)}\|_K
\left\| (u^{(3)}_i \cdots u^{(K)}_i)_{i=1}^d \right\|_{K/(K-2)}.
\end{align*}
Then by applying the same argument recursively we obtain the assertion.

The second assertion is obvious because $\|u\|_2 \leq \|u\|_K$ for all $K\geq 2$.
\end{proof}

\bibliographystyle{abbrvnat}
\bibliography{main} 


\end{document}